\documentclass[]{jair}
\usepackage[utf8]{inputenc}
\usepackage{mathtools}
\usepackage{ifthen}
\usepackage{braket}
\usepackage{amsthm}
\usepackage{url}
\usepackage{appendix}
\usepackage{hyperref}
\usepackage{subcaption}
\usepackage[shortlabels]{enumitem}
\usepackage{xcolor}
\usepackage{enumitem}
\usepackage{wrapfig}
\usepackage{footmisc}
\usepackage{amsmath}

\usepackage{amssymb}
\usepackage{bm}

\usepackage{etoolbox} 

\makeatletter
\newcommand{\pushright}[1]{\ifmeasuring@#1\else\omit\hfill$\displaystyle#1$\fi\ignorespaces}
\newcommand{\pushleft}[1]{\ifmeasuring@#1\else\omit$\displaystyle#1$\hfill\fi\ignorespaces}
\makeatother

\newcommand{\fun}[1]{\ensuremath{\mathopen{}\mathclose\bgroup\left(#1\aftergroup\egroup\right)}}

\newcommand{\vect}[1]{\ensuremath{\bm{#1}}}

\newcommand{\states}{\ensuremath{\mathcal{S}}}
\newcommand{\actions}{\ensuremath{\mathcal{A}}}

\newcommand{\probtransitions}{\ensuremath{\mathbf{P}}} 
\newcommand{\rewards}{\ensuremath{\mathcal{R}}}

\newcommand{\act}[1]{\ensuremath{\mathit{Act}\ifthenelse{\equal{#1}{}}{}{(#1)}}}

\newcommand{\policy}{\ensuremath{\pi}}

\newcommand{\stationary}[1]{\ensuremath{\xi_{#1}}}

\newcommand{\vecrewards}{\ensuremath{\vect{\rewards}}}

\newcommand{\momdp}{\ensuremath{\vect{\mathcal{M}}}}

\newcommand{\ld}{\ensuremath{\succ_L}}
\newcommand{\lld}[1][]{\ensuremath{\succ_{\lambda_{#1}}}}
\newcommand{\pd}{\ensuremath{\succ_P}}

\newcommand{\pf}[1]{\ensuremath{\mathcal{F}\left(#1\right)}}
\newcommand{\lf}[1]{\ensuremath{\mathcal{L}\left(#1\right)}}
\newcommand{\llf}[2]{\ensuremath{\mathcal{L}\left(#1; #2\right)}}

\newcommand{\observationfn}{\ensuremath{\mathcal{O}}}

 \newcommand{\encoderparameter}{\ensuremath{}}

\newcommand{\expectedsymbol}[1]{\ensuremath{\mathop{\mathbb{E}}\ifthenelse{\equal{#1}{}}{}{_{#1}}}}

\newcommand{\normal}[3]{\ensuremath{\displaystyle \ifthenelse{\equal{#3}{}}{\mathcal{N}(#1, #2)}{\mathcal{N}(#3\,;\, #1, #2)}}}

\newcommand{\overbar}[1]{\mkern 1.5mu\overline{\mkern-1.5mu#1\mkern-1.5mu}\mkern 1.5mu}
\newcommand{\overbarit}[1]{\,\overline{\!{#1}}}
\newcommand{\embed}{\ensuremath{\phi}}

\newcommand{\latentprobtransitions}{\ensuremath{\overbar{\probtransitions}}}

\newcommand{\latentrewards}{\ensuremath{\overbarit{\rewards}}}

\newcommand{\latentbeliefupdate}{\ensuremath{\overbar{\tau}}}

\newcommand{\localtransitionloss}[1]{L_{\probtransitions}}
\newcommand{\localrewardloss}[1]{L_{\rewards}}
\newcommand{\observationloss}[1]{\ensuremath{L_{\observationfn}}}
\newcommand{\beliefloss}[1]{\ensuremath{L_{\latentbeliefupdate}}}
\newcommand{\onpolicyrewardloss}[1]{\ensuremath{L_{\latentrewards}^{\varphi}}}
\newcommand{\onpolicytransitionloss}[1]{\ensuremath{L_{\latentprobtransitions}^{\varphi}}}

\newcommand{\KR}[1]{\ensuremath{\ifthenelse{\equal{#1}{}}{K_{\latentrewards}}{K_{\latentrewards}^{#1}}}}
\newcommand{\KP}[1]{\ensuremath{\ifthenelse{\equal{#1}{}}{K_{\latentprobtransitions}}{K_{\latentprobtransitions}^{#1}}}}

\newcommand{\originaltolatentstationary}[1]{{\latentprobtransitions_{\embed_{\encoderparameter}\stationary{\ifthenelse{\equal{#1}{}}{\policy}{#1}}}}}

\def\1{\bm{1}}

\def\vv{{\bm{v}}}

\DeclareMathAlphabet{\mathsfit}{\encodingdefault}{\sfdefault}{m}{sl}
\SetMathAlphabet{\mathsfit}{bold}{\encodingdefault}{\sfdefault}{bx}{n}

\DeclareMathOperator*{\argmax}{arg\,max}

\setcopyright{cc}
\copyrightyear{2026}
\acmDOI{10.1613/jair.1.19862}

\JAIRAE{Ivor Tsang}
\JAIRTrack{} 
\acmVolume{85}
\acmArticle{31}
\acmMonth{3}
\acmYear{2026}

\RequirePackage[
  datamodel=acmdatamodel,
  style=acmauthoryear,
  backend=biber,
  giveninits=true,
  uniquename=init
  ]{biblatex}

\usepackage{mathtools}
\usepackage{amsthm}
\usepackage{cleveref}
\usepackage{booktabs}
\usepackage{graphicx}
\usepackage[normalem]{ulem}
\useunder{\uline}{\ul}{}

\theoremstyle{plain}
\newtheorem{theorem}{Theorem}

\newtheorem{lemma}[theorem]{Lemma}

\theoremstyle{definition}
\newtheorem{definition}{Definition}

\theoremstyle{remark}

\newenvironment{proofsketch}{%
\renewcommand{\proofname}{Proof sketch}\proof}{\endproof}

\begin{document}

\title[Scalable Multi-Objective Reinforcement Learning]{Scalable Multi-Objective Reinforcement Learning with Fairness Guarantees using Lorenz Dominance}

\author{Dimitris Michailidis}
\authornote{Corresponding Author.}
\orcid{0000-0002-0106-1126}
\email{d.michailidis@uva.nl}
\affiliation{%
  \institution{University of Amsterdam}
  \city{Amsterdam}
  \state{Noord Holland}
  \country{Netherlands}
}

\author{Willem Röpke}
\orcid{0000-0001-5045-6127}
\email{willem.ropke@vub.be}
\affiliation{%
  \institution{Vrije Universiteit Brussel}
  \city{Brussels}
  \country{Belgium}}

\author{Diederik M.\ Roijers}
\orcid{0000-0002-2825-2491}
\email{diederik.roijers@vub.be}
\affiliation{%
  \institution{Vrije Universiteit Brussel \& City of Amsterdam}
  \city{Amsterdam}
  \state{Noord Holland}
  \country{Netherlands}
}

\author{Sennay Ghebreab}
\orcid{0009-0007-5788-4635}
\email{s.ghebreab@uva.nl}
\affiliation{%
  \institution{University of Amsterdam}
  \city{Amsterdam}
  \state{Noord Holland}
  \country{Netherlands}
}

\author{Fernando P. Santos}
\orcid{0000-0002-2310-6444}
\email{f.p.santos@uva.nl}
\affiliation{%
  \institution{University of Amsterdam}
  \city{Amsterdam}
  \state{Noord Holland}
  \country{Netherlands}
}

\renewcommand{\shortauthors}{Michailidis, Röpke, Roijers, Ghebreab \& Santos}

\begin{abstract}
Multi-Objective Reinforcement Learning (MORL) aims to learn a set of policies that optimize trade-offs between multiple, often conflicting objectives. MORL is computationally more complex than single-objective RL, particularly as the number of objectives increases. Additionally, when objectives involve the preferences of agents or groups, incorporating fairness becomes both important and socially desirable. This paper introduces a principled algorithm that incorporates fairness into MORL while improving scalability to many-objective problems. We propose using Lorenz dominance to identify policies with equitable reward distributions and introduce $\lambda$-Lorenz dominance to enable flexible fairness preferences. We release a new, large-scale real-world transport planning environment and demonstrate that our method encourages the discovery of fair policies, showing improved scalability in two large cities (Xi’an and Amsterdam). Our methods outperform common multi-objective approaches, particularly in high-dimensional objective spaces.
\end{abstract}

\received{17 July 2025}
\received[accepted]{12 February 2026}

\maketitle

\section{Introduction}
Reinforcement Learning (RL) is a powerful framework for sequential decision-making, where agents learn to maximize long-term rewards by interacting with an environment \cite{wang_deep_2020}. In most RL applications, rewards are constructed by aggregating multiple criteria (or objectives) into a single scalar value, typically via a weighted sum \cite{hayes_practical_2022}. However, this approach assumes prior knowledge of the precise preferences among objectives---a condition that rarely holds in real-world settings. Moreover, many real-world problems inherently involve multiple, often conflicting, objectives. Defining a scalar reward function before training can therefore introduce bias into the learning process, potentially excluding policies that differ primarily in their objective weightings \cite{vamplew_scalar_2022}.

Multi-Objective Reinforcement Learning (MORL) addresses this challenge by defining a separate reward function for each objective \cite{hayes_practical_2022}. This yields a set of candidate optimal policies, rather than a single solution, that decision-makers select according to their preferences. MORL has been successfully applied in various domains, such as decision-making under unknown preferences \cite{roijers_computing_2015,alegre_optimistic_2022}, human-value alignment \cite{peschl_moral_2021,rodriguez-soto_instilling_2022}, robot locomotion \cite{cao2021efficient}, and multi-agent systems \cite{radulescu_equilibria_2019,ropke_reinforcement_2023,vamplew_scalar_2022}.

Single-policy MORL learns one policy based on predefined knowledge about decision-makers' preferences. However, such preferences are not always known at training time. Multi-policy methods in MORL handle unknown preferences by assuming a monotonically increasing utility function and optimizing all objectives simultaneously, approximating the Pareto front of optimal policies \cite{mannion_multi-objective_2021,hayes_practical_2022,reymond_pareto_2022}. Multi-policy methods face scalability challenges as the solution set can scale exponentially with the number of objectives. This issue becomes particularly severe in \textit{many-objective} optimization, where the number of objectives is large \cite{nguyen_multi-objective_2020,perny_approximation_2013}. Consequently, multi-policy methods often struggle to scale efficiently in these scenarios, highlighting the need for further research in \textit{many-objective RL} \cite{hayes_practical_2022}.

Learning the entire Pareto front is often unnecessary, since some policies may be inherently undesirable \cite{osika2023what}. For example, in fairness-critical applications, some Pareto-non-dominated policies may result in unequal reward distributions across objectives. This is especially problematic when objectives represent the utilities of different societal groups \cite{pmlr-jabbari-fairness,cimpean_multi-objective_2023}. Although egalitarian approaches such as maxmin or equal weighting can address this issue, they assume a predefined, exact preference over objectives, often limiting flexibility and sometimes yielding inefficient results \cite{siddique2020learning}. This reveals a research gap in MORL: no current multi-policy method (a) guarantees fairness to the decision-maker, (b) allows control over fairness constraints, and (c) scales effectively to many-objective problems.

In this paper, we propose using Lorenz dominance to identify a subset of the Pareto front that ensures equitable reward distribution, without requiring predefined preferences. We extend this approach with $\lambda$-Lorenz dominance, enabling decision-makers to adjust the strictness of fairness constraints through a parameter $\lambda$. We formally show that $\lambda$-Lorenz dominance interpolates between Lorenz and Pareto dominance, providing decision-makers with fine-grained control over the degree of fairness. We also introduce Lorenz Conditioned Networks (LCN), a novel algorithm for optimizing $\lambda$-Lorenz dominance.

To support scalability in many-objective settings, we develop a new large-scale, multi-objective environment for planning transport networks in real-world cities with a flexible number of objectives. We conduct experiments in the cities of Xi’an (China) and Amsterdam (Netherlands) and show that LCN generates fair policy sets in large objective spaces. Since Lorenz-optimal sets are typically smaller than Pareto-optimal sets \cite{perny_approximation_2013}, LCN scales effectively, particularly in many-objective problems. We release the code and data used to generate our results alongside this paper \footnote{GitHub repository: \url{https://github.com/sias-uva/mo-transport-network-design}}.

\section{Related Work}
Our work intersects multi-policy methods for MORL \cite{hayes_practical_2022} and algorithmic fairness in sequential decision-making problems \cite{gajane_survey_2022}.
\subsection{Multi-Policy MORL}
Early multi-policy methods in RL, such as Pareto Q-learning, were limited to small-scale environments \cite{moffaert_multi-objective_2014,ruiz-montiel_temporal_2017,parisi2016multi}. To improve scalability, many approaches assume linear decision-maker preferences, resulting in a simpler solution set called the convex coverage set \cite{roijers_computing_2015,abels_dynamic_2019,felten_decomposition}. For example, GPI-LS---a state-of-the-art method and our baseline---decomposes the multi-objective problem into single-objective subproblems. Each subproblem uses a reward function that is a convex combination of the original vectorial reward. It then trains a neural network to approximate optimal policies for different weights \cite{alegre_gpi}. Beyond linear preferences, the Iterated Pareto Referent Optimization (IPRO) method uses a similar decomposition-based approach and provides strong theoretical guarantees. However, its performance degrades as the number of objectives increases \cite{ropke2024divide}. In contrast, Pareto Conditioned Networks (PCNs) do not decompose the problem, but instead train a return-conditioned policy, \cite{reymond_pareto_2022,reymond_exploring_2022,delgrange_wae-pcn_2023}. PCNs have been applied in various domains, including water management \cite{osika2025multi}, pandemic intervention policies \cite{chen2025learning}, autonomous cyber defence \cite{o2025multi} and battery control \cite{huoptimize}. Similar to PCN, PD-MORL approximates the Pareto front by uniformly sampling preferences across the preference space \cite{basaklar2022pd}, while C-MORL \cite{liu2024c} bridges constrained policy optimization and MORL to efficiently discover the Pareto front through parallel policy training. We propose a method inspired by PCNs that focuses on fairness, avoiding the need to search the entire preference space. This enables scalability to higher dimensions, learning a set of fair policies and offering flexibility in setting the degree of fairness preference. 

\subsection{Fairness in MORL}
Research on fairness in RL can be categorized along two main themes \cite{gajane_survey_2022}: fairness in domains where individuals belong to protected groups (societal bias) and fairness in resource allocation problems (non-societal bias). Our work aligns closely with the first theme, focusing on the fair distribution of benefits among different societal groups. Group fairness has been studied in RL before, specifically in multi-agent scenarios \cite{SatijaLPP23,ju2023achieving}, where agents learn individual policies. While we focus on single-agent RL, we assume that the agent's policies will affect groups of individuals, who may have conflicting preferences

Achieving fairness in RL often requires balancing multiple objectives. Many studies in this area incorporate diverse objectives into a single fairness-based reward function. This is typically achieved through linear reward scalarization \cite{rodriguez-soto_multi-objective_2021,chen_same-day_2023,blandin2024group}, nonlinear reward combinations, and welfare functions (e.g. the Generalized Gini Index) \cite{siddique2020learning,hu_towards_2023,fan_welfare_2023}, and other reward-shaping mechanisms \cite{zimmer_learning_2021,yu_policy_2022,mandal_socially_2023,kumar_fairness_2023}. Alternatively, some methods adapt the reward function during training to satisfy fairness constraints \cite{chen_bringing_2021}. These approaches require encoding fairness principles into the reward functions a priori, requiring preference information before training. Our method avoids these assumptions.

Our work is closely related to \cite{cimpean_multi-objective_2023}, which proposes a formal MORL fairness framework that encodes six fairness notions as objectives. The authors use PCNs to identify Pareto-optimal trade-offs among these fairness notions. While our method can be used within this framework, it differs by not predefining any specific fairness notion. Instead, it learns a set of non-dominated policies across all objectives, allowing the decision-maker to define their fairness criteria after training and select a policy accordingly. 

Our method relies on Lorenz dominance, a criterion that favors policies with balanced reward distributions. Lorenz dominance has previously been used in multi-objective optimization methods \cite{CHABANE20191,fasihi_bi-objective_2023,bederina2024generating}, but its application in MORL has been limited. \cite{perny_approximation_2013} first introduced the Lorenz criterion in Multi-Objective Markov Decision Processes (MOMDPs), providing much of the theoretical foundation we rely on this work. However, their experiments were limited to small-scale, randomly generated MOMDPs. Since then, some works inspired by the Lorenz curve have emerged; for example, \cite{siddique2020learning} uses the Generalized Gini Function to create a weighted sum in the Lorenz space for single-objective RL. Building on the framework of \cite{perny_approximation_2013}, we train a neural network to learn the full Lorenz front, enabling flexible degrees of fairness for the decision-maker and demonstrating scalability to significantly larger and more realistic environments.

\section{Preliminaries}
In this section, we formally introduce Multi-Objective Reinforcement Learning (MORL) and Lorenz dominance.
\subsection{Multi-Objective Reinforcement Learning}

We consider reinforcement learning agents that interact with a Multi-Objective Markov Decision Process (MOMDP). An MOMDP is represented as a tuple $\momdp = \langle \states, \actions, \probtransitions, \vecrewards, \gamma \rangle$ consisting of a set of states $\states$, set of actions $\actions$, transition function $\mathcal{T}: \mathcal{S} \times \mathcal{A} \times \mathcal{S}  \to [0,1]$, vector-based reward function $\vecrewards: \states \times \actions \times \states \to \mathbb{R}^d$, with $d \geq 2$ the number of objectives, and a discount factor $\gamma \in [0, 1)$. In an MOMDP, we consider deterministic policies $\policy: \states \to \actions$, which map states to actions.

With vector-based rewards, there is generally no single optimal policy in the usual sense of scalar rewards (e.g. policy maximizing reward). Instead, we learn a set of optimal policies using a dominance criterion. In MORL, Pareto dominance is commonly applied, resulting in a solution set called the Pareto front.

\begin{definition} (Pareto dominance)
Consider two vectors $\vv, \vv' \in \mathbb{R}^d$. We say that $\vv$ Pareto dominates $\vv'$, denoted $\vv \succ_P \vv'$, when $\forall j \in \{1, \dotsc, d\}: v_j \geq v'_j$ and $\vv \neq \vv'$.
\end{definition}

In essence, $\vv$ Pareto dominates $\vv'$ when it is at least as good in all objectives and strictly better in at least one. For a set of vectors $D$, the Pareto front $\pf{D}$ contains all non-Pareto-dominated vectors. 

\begin{definition} (Pareto front)
Given a set of vectors $D \subseteq \mathbb{R}^d$, the Pareto front $\pf{D}$ is the subset of vectors in $D$ that are not Pareto-dominated by any other vector in $D$. Formally,
\[
\pf{D} = \{ \vv \in D \mid \nexists \; \vv' \in D \text{ such that } \vv' \succ_P \vv \}.
\]
\end{definition}

\subsection{Fairness in Many-Objective Reinforcement Learning}
\label{subsec:fairness_mo}
One common fairness approach in Multi-Objective Reinforcement Learning (MORL) is to treat all objectives as equally important, optimizing a single, equally weighted objective. However, this assumes absolute equality in reward distribution, which may be infeasible in certain problems, and yields a single policy, without offering options to the decision-maker. Another approach, inspired by Rawlsian justice theory and the \textit{maxmin} principle, focuses on maximizing the minimum reward between objectives. However, this often results in solutions that are not efficient for all users \cite{siddique2020learning}.

To train a multi-policy algorithm with fair trade-offs, it is necessary to identify all optimal trade-offs that achieve a fair distribution of rewards. To achieve this, we use Lorenz dominance, a refinement of Pareto dominance that considers the distribution of values within a vector \cite{perny_approximation_2013}. This concept, traditionally used in economics to assess income inequality \cite{ranking_income}, is adapted here for fairness in MORL.

\begin{definition} (Lorenz dominance)
Let $L(\vv)$ be the Lorenz vector of a vector $\vv \in \mathbb{R}^d$, defined as follows:
\begin{equation}
L(\vv) = \Bigl( v_{(1)}, v_{(1)} + v_{(2)}, \cdots , \sum_{i=1}^d v_{(i)} \Bigl), 
\end{equation}
where $v_{(1)} \leq v_{(2)} \leq \cdots \leq v_{(d)}$ are the values of the vector $\vv$, sorted in increasing order. A vector $\vv$ Lorenz dominates a vector $\vv'$, when its Lorenz vector $L(\vv)$ Pareto dominates the Lorenz vector $L(\vv')$ \cite{perny_approximation_2013}. We use $\vv \ld \vv'$ to denote that $\vv$  Lorenz dominates $\vv'$.
\end{definition}

For a set of vectors $D$, the Lorenz front $\lf{D}$ contains all vectors that are non-Lorenz-dominated. 

\begin{definition} (Lorenz front)
Given a set of vectors $D \subseteq \mathbb{R}^d$, the Lorenz front $ \lf{D} $ is the subset of vectors in D that are not Lorenz-dominated by any other vector in D. Formally,
\[
\lf{D} = \{ \vv \in D \mid \nexists \; \vv' \in D \text{ such that } L(\vv') \succ_P L(\vv) \},
\]
where \( L(\vv) \) denotes the Lorenz vector of \( \vv \), and $\succ_P$ is the Pareto dominance relation.
\end{definition}

Our approach builds on the Pigou-Dalton transfer principle from economics \cite{perny_approximation_2013,adler_pigou-dalton_2013}. This principle states that a redistribution of value from a better-off to a worse-off component improves fairness, as long as it does not reverse their ranking. Formally, given a reward vector $\vv \in \mathbb{R}^d$ with two components $v_i > v_j$ for some indices $i$ and $j$, transferring a small amount $\epsilon$ ($0 < \epsilon \leq v_i - v_j$) from $v_i$ to $v_j$ yields a new vector $\vv' = \vv - \epsilon I_{i} + \epsilon I_{j}$, where $I_{i} \in \mathbb{R}^d$ is an \textit{indicator vector} with a 1 in the $i$th position and 0 elsewhere (and similarly for $I_{j}$). This transformation preserves the total reward and ranking, but results in a more equitable distribution, and is therefore called a Pigou-Dalton transfer \cite{perny_approximation_2013}.

Lorenz-based fairness evaluates policies based on how equitably rewards are distributed across objectives. A solution is considered fairer if it can be obtained via a sequence of Pigou-Dalton transfers from another, implying that it Lorenz-dominates the other. For example, consider $\vv = (8, 0)$. A transfer of $\epsilon = 3$ yields $\vv' = (5, 3)$. Even though the total reward remains 8, and the rank between entries is preserved, $\vv'$ is considered fairer under Lorenz dominance.

In MORL, we define vectors $\vv^\pi, \vv^{\pi'}$ as the expected return of the policies $\pi, \pi'$, across all objectives of the environment, respectively. We define fair policies as those that are non-Lorenz-dominated. The set of non-dominated value vectors is called a \textit{Lorenz coverage set}, which is usually (but not necessarily) significantly smaller than a Pareto coverage set \cite{perny_approximation_2013}. Our fairness approach satisfies the criteria outlined in \cite{siddique2020learning}. It is Lorenz-\textit{efficient} as the learned policies are non-Lorenz-dominated; it is \textit{impartial}, since it treats all objectives as equally important, and it is \textit{equitable}, as Lorenz dominance satisfies the Pigou-Dalton principle \cite{perny_approximation_2013}. In \Cref{fig:pd_vs_ld}, we show the difference between Pareto and Lorenz dominance. The latter extends the area of undesired solutions, allowing for fewer non-dominated solutions and providing fairness guarantees in the two objectives.

\begin{figure}[t!]
\begin{center}
\centerline{\includegraphics[width=4.5in]{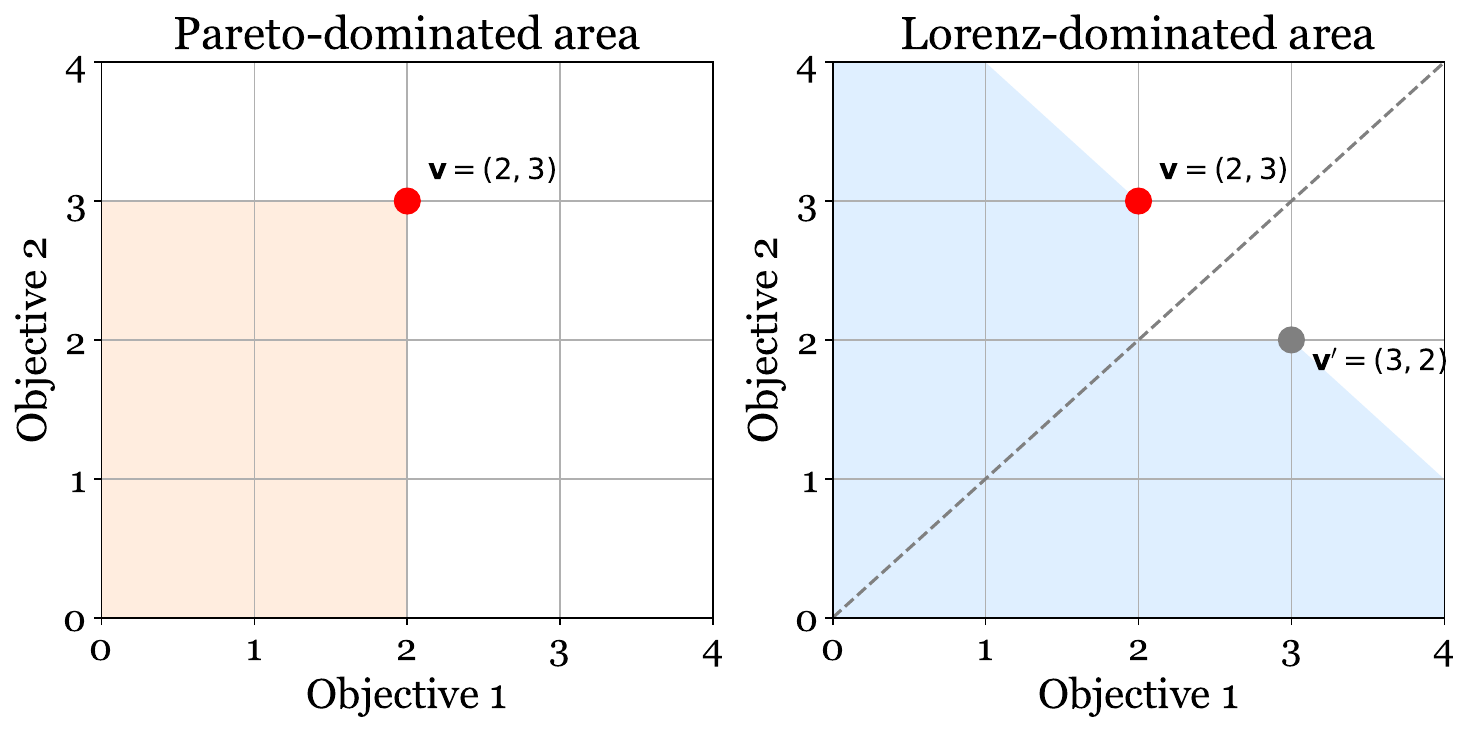}}
\Description{Illustration of the Pareto-dominated and Lorenz-dominated areas of vector \vv. The Lorenz-dominated area includes the Pareto-dominated area and is symmetric around the equality line, except for the symmetric vector \vv' = (3, 2), which expands the dominance and reduces the number of acceptable trade-offs.}
\caption{The Pareto and Lorenz-dominated areas of vector $\vv$. The Lorenz-dominated area includes the Pareto-dominated area, and is symmetric around the equality line, except for the symmetric vector $\vv' = (3, 2)$. This creates an expanded dominance, resulting in fewer acceptable trade-offs.}
\label{fig:pd_vs_ld}
\end{center}
\end{figure}

\section{Flexible Fairness with \texorpdfstring{$\lambda$}{lambda}-Lorenz Dominance}
\label{subsec:lambda_lcn}
To give decision-makers fine-grained control over the fairness needs of their specific problem, we introduce a novel criterion, called $\lambda$-Lorenz dominance. $\lambda$-Lorenz dominance operates directly on the return vectors, without objective weights (such as the Generalized Gini Index \cite{siddique2020learning}). By selecting a single parameter $\lambda \in [0, 1]$, $\lambda$-Lorenz dominance allows decision-makers to balance Pareto and Lorenz optimality. We formally define $\lambda$-Lorenz dominance in \Cref{def:lld}.

\begin{definition}[$\lambda$-Lorenz dominance]
\label{def:lld}
Let $\sigma(\vv)$ be the vector $\vv$ sorted in increasing order. Given $\lambda \in [0, 1]$, a vector $\vv$ $\lambda$-Lorenz dominates another vector $\vv'$, denoted $\vv \lld \vv'$ if,
\begin{equation}
    \lambda \sigma(\vv) + (1-\lambda) L(\vv) \succ_P \lambda \sigma(\vv') + (1-\lambda) L(\vv').
\end{equation}
\end{definition}

Intuitively, for $\lambda=1$, it is assumed that the decision-maker cares equally about all objectives and thus may reorder them. This relaxation allows some non-Pareto-dominated vectors to become dominated. Consider, for example, $(4, 2)$ and $(1, 3)$. While no vector is Pareto-dominated, reordering the objectives in increasing order yields $(2, 4)$ and $(1, 3)$, in which case the second vector is now dominated. This approach may already reduce the size of the Pareto front, but does not yet achieve the same fairness constraints imposed by the Lorenz front. At the other extreme, setting $\lambda=0$ ensures that the solution set is equal to the Lorenz front.

The $\lambda$-Lorenz front of a set $D$, denoted $\llf{D}{\lambda}$, contains all vectors that are pairwise non-$\lambda$-Lorenz-dominated. In \Cref{th:front-relations}, we formally show that the $\lambda$-Lorenz fronts form increasing nested sets as $\lambda$ varies from $0$ to $1$: lower $\lambda$ values mean increasingly selective coverage sets that interpolate between the Pareto and Lorenz fronts. The full proof is provided in the supplementary material.

\begin{theorem}
\label{th:front-relations}
$\forall \lambda_1, \lambda_2: 0 \leq \lambda_1 \leq \lambda_2 \leq 1$ and $\forall D \subset \mathbb{R}^d$ the following relations hold.
\begin{equation}
    \lf{D} \subseteq \llf{D}{\lambda_1} \subseteq \llf{D}{\lambda_2} \subseteq \pf{D}.
\end{equation}
\end{theorem}
\begin{proofsketch}
To prove \Cref{th:front-relations}, we provide three auxiliary results. First, we demonstrate that for all parameters $\lambda \in [0, 1]$ and vectors $\vv, \vv' \in \mathbf{R}^d$ we have that:
\begin{equation}
    \vv \lld \vv' \implies \vv \ld \vv'. 
\end{equation}
Together with some algebra, this result is subsequently used to show that for all parameters $\lambda_1$ and $\lambda_2$ such that $0 \leq \lambda_1 \leq \lambda_2 \leq 1$, 
\begin{equation}
    \vv \lld[2] \vv' \implies \vv \lld[1] \vv'. 
\end{equation}
Finally, we extend \cite[Theorem~1]{perny_approximation_2013} to show that Pareto dominance implies $\lambda$-Lorenz dominance as well. These components are combined to obtain the desired result.
\end{proofsketch}

It is a straightforward corollary that for $\lambda = 1, \llf{D}{\lambda} \subseteq \pf{D}$ while for $\lambda = 0, \llf{D}{\lambda} = \lf{D}$. In \Cref{fig:lcn_models} (C) we illustrate conceptually how the $\lambda$ controls the size of the coverage set to consider.

\section{Lorenz Conditioned Networks}
Lorenz Condition Networks (LCNs) are an adaptation of Pareto Conditioned Networks (PCNs) that aim to learn policies on the $\lambda$-Lorenz front. We use both the abbreviations PCN and PCNs (and likewise LCN and LCNs) interchangeably, depending on whether we refer to the framework as a whole (singular) or to the family of networks (plural). We begin this section by giving an overview of PCNs and identifying their drawbacks.

\subsection{Background: Pareto Conditioned Networks (PCN)}
A PCN is a supervised learning network designed for Multi-Objective Reinforcement Learning (MORL) \cite{reymond_pareto_2022}. It enables a single neural network to learn a diverse set of policies that approximates the Pareto front of optimal trade-offs between objectives. A PCN is a conditioned network, meaning that it is trained to take as input both the environment state and a desired return vector, and output a probability distribution over actions.

Formally, the PCN policy is parameterized as: $\pi_\theta(a_t \mid s_t, \hat{h}_t, \hat{\boldsymbol{R}}_{t})$, where $s_t$ is the current state, $\hat{h}_t$ is the desired horizon, and $\hat{\boldsymbol{R}}_{t}$ is a desired return vector over the $d$ objectives. The network is trained using supervised learning on transitions collected during exploration. Each transition includes the state, the action taken, and the return vector obtained from the trajectory, allowing the network to learn by imitating non-dominated behavior conditioned on return goals.

During training, PCN incrementally improves the quality of collected experiences by perturbing non-dominated return vectors in the experience replay (ER) buffer. A new desired return $\hat{\boldsymbol{R}}_{t}$ is sampled by adding noise to an existing non-dominated point (proportional to the standard deviation of current returns), which serves as a target for generating new exploratory trajectories. This process helps expand the diversity of collected policies and encourages coverage of the Pareto front.

To filter the ER buffer, PCN introduces a filtering heuristic that favors experiences closer to the Pareto front while preserving diversity. This is achieved by computing the Euclidean distance of each experience to the Pareto front (to prioritize proximity), and a crowding distance \cite{deb_fast_2000} (to encourage spread in objective space), and then applying penalties to overrepresented areas.

Despite its effectiveness in approximating the Pareto front, PCN suffers from two key drawbacks:
\begin{enumerate}
    \item Lack of fairness control --- PCN prioritizes Pareto optimality without considering fairness or equitable reward distribution across objectives, potentially favoring extreme, imbalanced solutions and leading to intractable learning in large state and objective spaces.
    \item Experience Replay volatility --- As new experiences are collected, many older points are replaced or re-evaluated, causing instability during training.
\end{enumerate}

LCNs belong to the same family of reward conditioned networks, also referred to as upside-down reinforcement learning, where a policy is trained as a single neural network through supervised learning \cite{kumar_reward-conditioned_2019,reymond_pareto_2022}. An LCN network learns multiple policies, each representing a Lorenz-optimal trade-off.

\subsection{LCN Network} Just like PCN, LCN uses a single neural network to learn a policy $\pi_\theta(a_t|s_t, \hat{h_t},\
{\hat{\boldsymbol{R}}}_{t})$, which maps the current state $s_t$, the desired horizon $\hat{h}_t$ and the desired return $\hat{\boldsymbol{R}}_{t}$ to the next action $a_t$. Note that $\hat{\boldsymbol{R}}_{t}$ is a vector with dimension $d$ equal to the number of objectives. The network receives an input tuple $\langle s_t, \hat{h}_t, \hat{\vect{R}}_t \rangle$ and returns a probability distribution over the set of potential next actions. It is trained with supervised learning on samples collected by the agent during exploration. The network updates its parameters using a cross-entropy loss function: 

\begin{equation}
    \label{eq:cross_entropy}
    \mathcal{J}(y,\pi) = -\frac{1}{N} \sum_{i=1}^{N} \sum_{a \in A} y_{a}^{(i)} \log \pi \left(a_t^{(i)} | s_t^{(i)}, h_t^{(i)}, \boldsymbol{R}_{t}^{(i)} \right),
\end{equation}
where $N$ is the batch size, $y_{a}^{(i)}$ is the $i$-th sample action taken by the agent (ground truth), $y_{a}^{(i)} = 1$ if $a_t = a$ and $0$ otherwise and $\pi(a_t^{(i)} | s_t^{(i)}, h_t^{(i)}, \boldsymbol{R}_{t}^{(i)})$  represents the predicted probability of action $a$  for the  $i$-th sample, conditioned on its specific state  $s_t^{(i)}$, horizon  $h_t^{(i)}$, and return  $\boldsymbol{R}_{t}^{(i)}$.
 
The training process involves sampling collected, non-dominated experiences and then training the policy with supervised learning to imitate these experiences. Given a sufficient number of good experiences, the agent will learn good policies.

\subsection{Collecting Experiences}
LCN learns the policy network $\pi_{\theta}$ by collecting experiences and storing them in an ER buffer. These experiences are then used to train the policy via supervised learning \cite{kumar_reward-conditioned_2019,reymond_pareto_2022}. Because action selection involves conditioning the network on a specified return, the primary mechanism for collecting higher-quality experiences is to iteratively improve the desired return used as the conditioning input.

To achieve this improvement, similarly to PCN, a non-dominated return is randomly sampled from the current non-dominated experiences in the ER buffer. This sampled return is then increased by a value drawn from a uniform distribution $U(0,\sigma_o)$, where $\sigma_o$ represents the standard deviation of all non-dominated points in the ER buffer \cite{reymond_pareto_2022}. The updated return is subsequently used as the input $\boldsymbol{\hat{R_t}}$ for the policy network. Through this iterative process---refining the condition, collecting improved experiences, and training the policy network on non-dominated experiences---the network progressively learns to approximate all non-dominated trade-offs, forming a Lorenz coverage set. To ensure that the ER buffer contains experiences that will contribute most to performance improvement, however, the buffer must be filtered to retain only the most useful experiences.

\subsection{Filtering Experiences} 
\label{subsec:filtering_experiences}
PCN improves the experience replay buffer by filtering out experiences that are far away from the currently approximated Pareto front. This is done by calculating the distance of each collected experience to the closest non-Pareto-dominated point in the buffer. In addition, a \textit{crowding distance} is calculated for each point, measuring its distance to its closest neighbors \cite{deb_fast_2000}. Points with many neighbors have a high crowding distance and are penalized, ensuring that ER experiences are distributed throughout the objective space \cite{reymond_pareto_2022}.
With this approach, the set of Pareto-optimal solutions can grow exponentially with the number of states and objectives \cite{perny_approximation_2013}, and maintaining a good ER buffer becomes a great challenge. This can be intuitively understood by a simple example provided in \cite{perny_approximation_2013}: consider a deterministic MOMDP with $N+1$ states, where each non-terminal state allows two actions with different two-objective reward vectors. The terminal state is absorbing, and assume a discount factor $\gamma = 1$. There are $2^{N+1}$ possible stationary deterministic policies, and exactly $2^N$ of them result in distinct value vectors at the initial state. These vectors take the form $(x, 2^N - 1 - x)$ for $x = 0, 1, \dots, 2^N - 1$, and all lie on the Pareto front since improving one objective necessarily worsens the other. Therefore, the number of Pareto-optimal value vectors grows exponentially with $N$, the number of non-terminal decision points.

This explosion is a consequence of the definition of Pareto dominance, which treats all objectives as equally important and assumes no prior knowledge of user preferences. Any improvement in any objective is valued, resulting in a broad set of solutions. However, in many practical applications, such a solution set may be unnecessary or even undesirable. For instance, when objectives correspond to benefits for distinct individuals or groups (and we know that fairness is preferred), some Pareto-optimal solutions may be undesirable, as they result in highly unequal outcomes.

It is important to note that this issue is not exclusive to Pareto dominance; other dominance relations, such as Lorenz dominance, can theoretically also produce exponentially large solution sets. Empirically, for a fixed number of objectives, approximating a Lorenz-optimal set within an error bound from the Lorenz front can be computed in polynomial time on the size of the state and objective spaces. We provide empirical results supporting this claim in \Cref{sec:experiment_setup}. This aligns with the intuition that stricter dominance criteria, such as those that prioritize fairness, tend to reduce the number of acceptable solutions by implicitly encoding preference information without relying on explicit weights.
This distinction is crucial for experience filtering: when prior knowledge about user preferences (e.g., fairness) is available, employing stricter dominance criteria can produce more computationally manageable ER buffers.

In Lorenz Conditioned Networks (LCN), the evaluation of each experience $e_i$ in the Experience Replay buffer $\mathcal{B}$ is determined by its proximity to the nearest \textit{non-Lorenz-dominated} point $l_j \in \lf{\mathcal{B}} \subseteq \mathcal{B}$. Here, the proximity between experiences is measured in the objective space, where each sampled experience $e_i$ is represented by its actual reward vector (not the condition used to generate the experience). Thus, the nearest non-Lorenz-dominated point is the one with the smallest distance in this space. In \Cref{fig:lcn_models} (A) we show an example of this distance calculation. To ensure meaningful distance comparisons across potentially differently scaled objectives, all objective vectors in the buffer are normalized to the $[0,1]$ range before computing distances. We denote the distance between an experience $e_i$ and a reference point $t_i$ as $d_{e_i, t_i} = || e_i - t_i||_2$, where $t_i = \min || e_i - l_j ||_2$ is the nearest non-Lorenz-dominated point, which we refer to as \textit{reference point}. We formalize the final distance for the evaluation $d_{\text{Lorenz}, i}$ as follows:
\begin{equation}
    \label{eq:er_filter}
    d_{\text{Lorenz}, i} = \begin{cases}
        d_{e_i, t_i} & \text{if } d_{cd, i} > \tau_{cd} \\
      \rho_{pen}(d_{e_i, t_i} +c) & \text{if } d_{cd, i} \leq \tau_{cd}
    \end{cases}
\end{equation}
Where $d_{cd}$ is the crowding distance of $i$ and $\tau_{cd}$ is the crowding distance threshold. A constant penalty $c$ is added to the points below the threshold, whose distance is additionally penalized by a penalty multiplier $\rho_{pen}$. For the experiments in this paper, we set $\rho_{pen} = 2$ and provide additional sensitivity analysis on \Cref{subsec:appendix:cd_sensitivity_analysis}.
The points in the ER buffer are sorted based on $d_{\text{Lorenz}}$, and those with the highest get replaced first when a better experience is collected. The threshold $\tau_{cd}$ separates serves to break high-density regions in the objective space and ensures diversity. Experiences with a crowding distance below this threshold lead to too similar policies and are therefore penalized. While this heuristic introduces a discontinuity at the threshold $\tau_{cd}$, in practice this acts as a strong bias toward maintaining diversity in the buffer, which benefits stable training.

\begin{figure}[t!]
\begin{center}
\centerline{\includegraphics[width=5.5in]{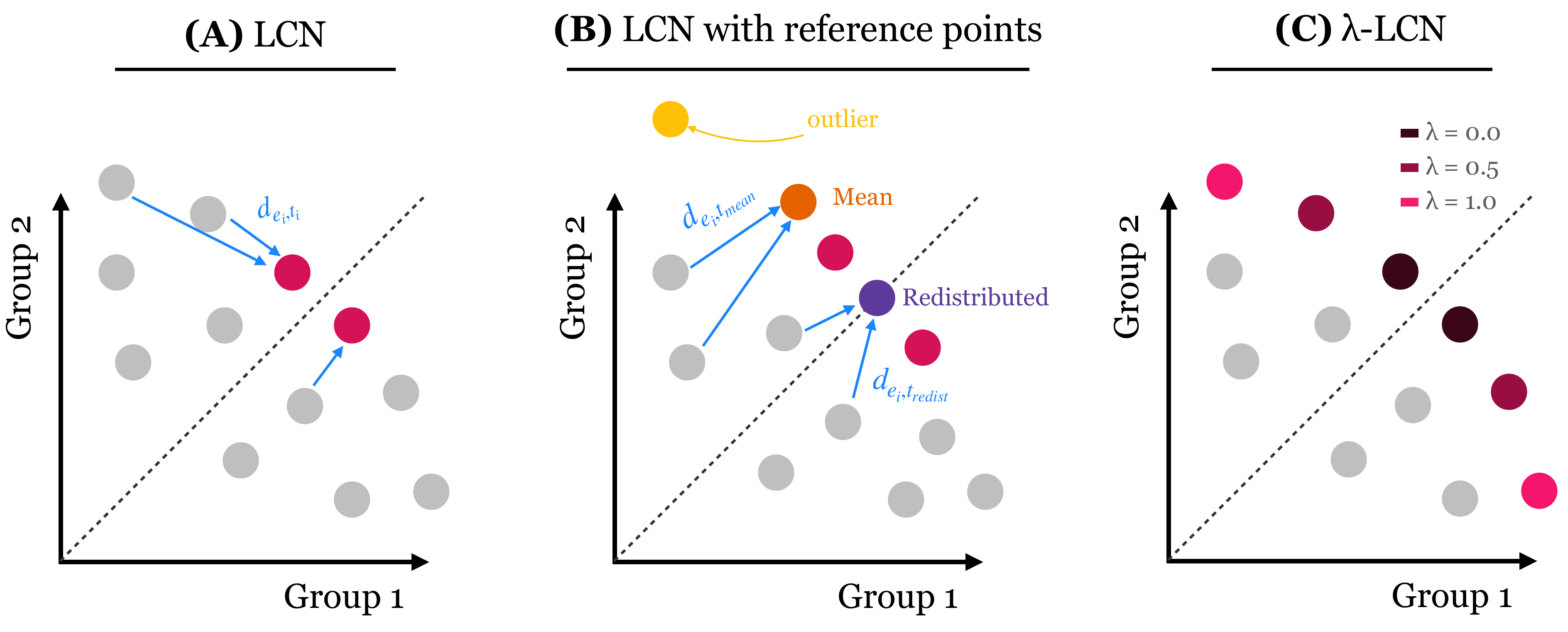}}
\Description{Illustration of Lorenz Conditioned Networks (LCNs) as a multi-policy method for fair trade-offs between objectives. Panel A shows standard LCNs balancing objectives by exploring the trade-off space. Panel B shows how reference points guide learning by filtering the Experience Replay buffer and reducing outlier influence. Panel C shows lambda-LCNs, which introduce flexibility in fairness preferences to allow more diverse policies. Each group corresponds to an objective related to group satisfaction, such as satisfied transportation demand.}
\caption{Lorenz Conditioned Networks (LCNs) is a multi-policy method that offers fair trade-offs between different objectives. In this illustrative example, each group (Group 1 and Group 2) corresponds to a distinct objective related to group satisfaction (e.g. satisfied transportation demand). \textbf{(A)} Standard LCNs learn policies that balance objectives by exploring the trade-off space between them. \textbf{(B)} Reference points help accelerate training by filtering the Experience Replay buffer and guiding learning toward desirable solutions, while reducing the influence of outliers. \textbf{(C)} $\lambda$-LCN introduces flexibility in the fairness preferences, enabling the relaxation of fairness constraints to accommodate more diverse policies.}
\label{fig:lcn_models}
\end{center}
\end{figure}

\subsection{Improving the Filtering Mechanism with Reference Points} 
\label{subsec:ref_points}
The nearest-point filtering method employed by previous works has two drawbacks. Firstly, during exploration, stored experiences undergo significant changes as the agent discovers new, improved trajectories. This leads to a volatile ER buffer and moving targets, posing stability challenges during supervised learning. Secondly, we know in advance that certain experiences, even if non-dominated, are undesirable due to their unfair distribution of rewards.

Consider, for example, vectors: $\vv=(8, 0), \boldsymbol{w} = (3, 4)$ and their corresponding Lorenz vectors $L(\vv) = (0, 8), L(\boldsymbol{w})=(3,7)$. Both $\vv$ and $\boldsymbol{w}$ are non-Lorenz-dominated, and would typically be used as targets for evaluating other experiences. However, $\vv$ is not a desirable target due to its unfair distribution of rewards (this is essentially a limitation of Lorenz dominance, when one objective is very large). To address these issues, we adopt the concept of reference points from Multi-Objective Optimization \cite{deb_moo_refpoints,cheng_ref_vector,felten_decomposition}, and propose reference points for filtering the experiences. We introduce two reference point mechanisms: a \textbf{redistribution} mechanism and a \textbf{mean} reference point mechanism. Both of these reference points are optimistic, meaning that the agent seeks to minimize the distance to them.

\subsubsection{Redistributed Reference Point (LCN-Redist)}
\label{subsec:redist}
This mechanism draws inspiration from the Pigou-Dalton principle we introduced in \Cref{subsec:fairness_mo}. Under this axiomatic principle, any experience in the ER buffer can be adjusted to provide a more desirable one. We identify the experience with the highest sum of rewards and evenly distribute the total reward across all dimensions of the vector. This is then assigned as the new \textit{reference point} $t$ for all experiences $e \in \mathcal{B}$:

\begin{equation}
    t_\text{redist} = \frac{1}{n} \sum_{i=1}^{n} \left( \underset{e \in D}{\mathrm{\argmax}} \sum_{j=1}^{n} e_j \right)_i.
\end{equation}

Note that $t$ is now the same for all $e \in \mathcal{B}$. Subsequently, we measure the distances of all $e \in \mathcal{B}$ to this reference point and filter out those farthest from it, according to Equation \ref{eq:er_filter} (replace $t_i$ with $t_\text{redist}$). In \Cref{fig:lcn_models} (B), we illustrate this transfer mechanism.

\subsubsection{Non-Dominated Mean Reference Point (LCN-Mean)}

Additionally, we propose an alternative, simpler reference point mechanism: a straightforward averaging of all non-Lorenz-dominated vectors in the experience replay (ER) buffer. This approach provides a non-intrusive method for incorporating collected experiences, while simultaneously smoothing out outlier non-dominated points. The reference point, denoted as $t_{\text{mean}}$, is defined as:

\begin{equation}
    t_{\text{mean}} = \frac{1}{|\lf{B}|} \sum_{l_j \in \lf{B}} l_j,
\end{equation}

where $\lf{B} = \{l_1, l_2, \ldots, l_j\}$ represent the set of non-Lorenz-dominated experiences in the ER buffer. In \Cref{fig:lcn_models} (B) we show how this approach defines a reference point. 

\section{A Large Scale Many-Objective Environment}
\label{sec:motndp}

Existing discrete MORL benchmarks are often small-scale, with small state-action spaces or low-dimensional objective spaces \cite{vamplew_empirical_2011,lopez_microscopic_2018}. In addition, they do not cover allocation of resources such as public transport, where fairness in the distribution is crucial. We introduce a novel and modular MORL environment, named the Multi-Objective Transport Network Design Problem (MO-TNDP). Built on MO-Gymnasium \cite{Alegre+2022bnaic}, the MO-TNDP environment simulates public transport design in cities of varying sizes and morphologies, addressing TNDP, an NP-hard optimization problem aiming to generate a transport line that maximizes the satisfied travel demand \cite{farahani2013review}.

In MO-TNDP, a city is represented as $H^{m\times n}$, a grid with equally sized cells. The mobility demand forecast between cells is captured by an Origin-Destination (OD) flow matrix $OD^{|H| \times |H|}$. Each cell $h \in H^{n \times m}$ is associated with a socioeconomic group $g \in \mathcal{R}$, which determines the dimensionality of the reward function. In this paper, we scale it from 2 to 10 groups (objectives).

\begin{figure}[ht]
\begin{center}
\centerline{\includegraphics[width=4.5in]{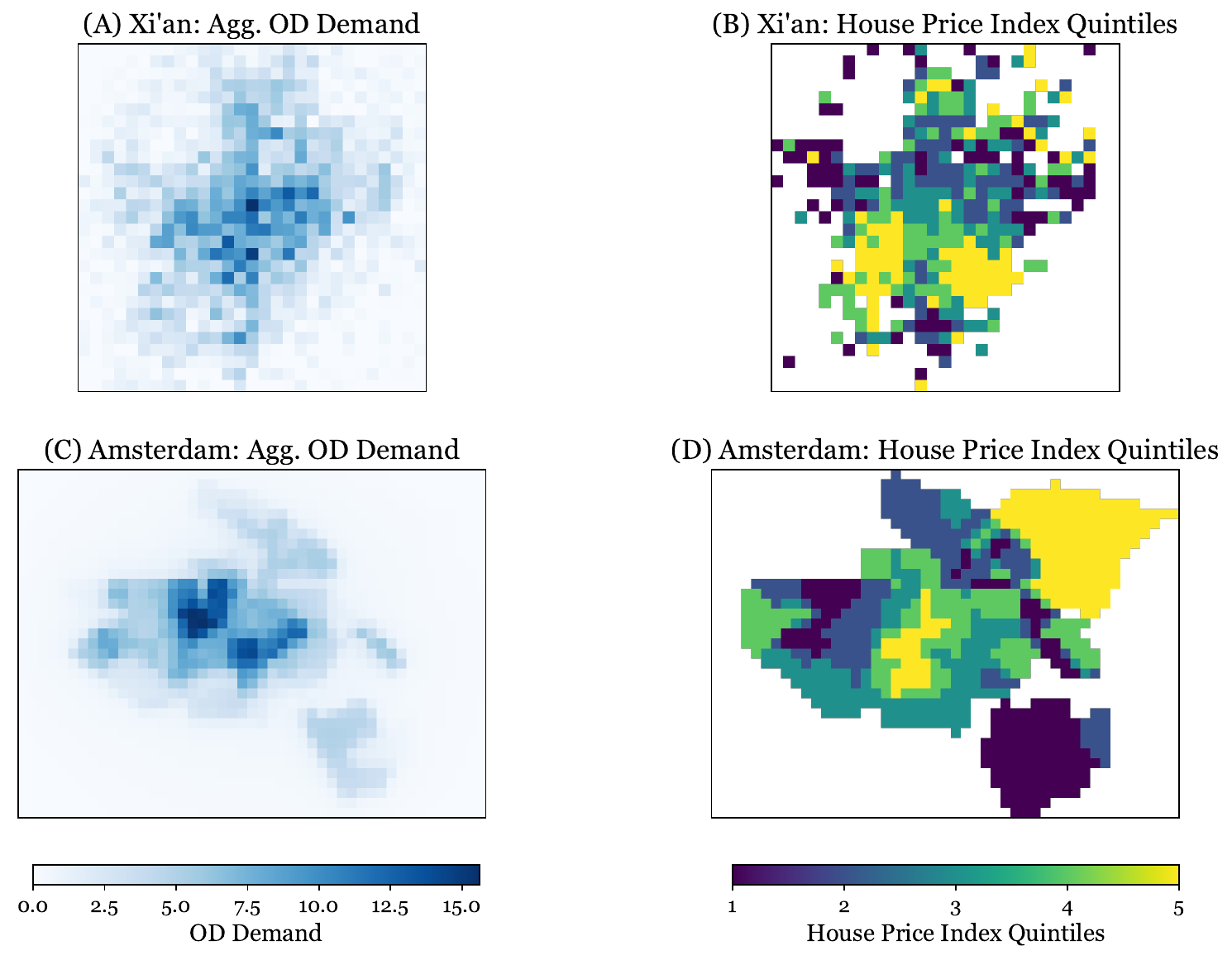}}
\Description{Two real-world instances of the MO-TNDP environment: Xi'an, China and Amsterdam, Netherlands. Panels A and C show aggregate Origin-Destination demand per cell for Xi'an (A) and Amsterdam (C). Panels B and D show group membership of each cell based on house price index quintiles for Xi'an (B) and Amsterdam (D).}
\caption{Two real-world instances of the MO-TNDP environment in Xi'an (China) \cite{wei_city_2020} and Amsterdam (Netherlands). (A) and (C) show the aggregate Origin-Destination Demand per cell (sum of incoming and outgoing flows) for Xi'an (A) and Amsterdam (C). (B) and (D) show the group membership of each cell, based on the house price index quintiles for Xi'an (B) and Amsterdam (D).}
\label{fig:tndp}
\end{center}
\end{figure}

Episodes last a predefined number of steps. The agent traverses the city, connecting grid cells with eight available actions (movement in all directions). At each time step, the agent receives a vectorial reward of dimension $|\boldsymbol{\mathcal{R}}|$, each corresponding to the percent satisfaction of the demand of each group. We formulate it as an MOMDP $\momdp = \langle \states, \actions, \probtransitions, \vecrewards, \gamma \rangle$, where $\mathcal{S}$ is the current location of the agent, $\mathcal{A}$ is the next direction of movement, and $\boldsymbol{\mathcal{R}}: \mathcal{S} \times \mathcal{A} \times \mathcal{S} \rightarrow \mathbb{R}^d$ is the additional demand satisfied by taking the last action for each group.

Given the discrete and episodic nature of this particular problem, we set the discount factor $\gamma$ to 1. The transition function $\probtransitions$ is deterministic, and each episode starts in the same state.

Additional directional constraints can be imposed on the agent's action space. The environment code enables developers to modify the city object, incorporating adjustments to grid size, OD matrix, cell group membership, and directional constraints, making it adaptable to any city. It supports both creating transport networks from scratch and expanding existing ones. MO-TNDP is available online\footnote{Githup repository: \url{https://github.com/dimichai/mo-tndp}}. 

\section{Experiments}
\label{sec:experiment_setup}
To evaluate our methods and compare them against baselines, we conduct experiments in two environments. One of these is the \textit{Deep Sea Treasure} environment --- a widely used benchmark in Multi-Objective Reinforcement Learning (MORL) \cite{vamplew_empirical_2011}. In this environment, an agent pilots a submarine through a grid-based ocean map, aiming to collect treasures located at various depths. The agent must navigate a trade-off between two conflicting objectives: maximizing treasure value and minimizing fuel consumption. Movement consumes fuel, while treasures located deeper in the sea yield higher rewards. This setup creates a trade-off between short, low-value paths and longer, more rewarding ones. Notably, because the time objective is modeled as a negative reward (i.e., a cost), only the baseline Linear Combination of Normals (LCN) method is applicable in our experiments.

In MO-TNDP, which we described in \Cref{sec:motndp}, we run experiments on two cities: Xi'an in China (841 cells, 20 episode steps) and Amsterdam in the Netherlands (1645 cells, 10 episode steps). Group membership for cells is determined by the average house price, which is divided into 2-10 equally sized buckets. \Cref{fig:tndp} illustrates two instances of the MO-TNDP environment for five objectives (a map of group membership for all objectives is provided in the supplementary material). LCNs are built using the MORL-Baselines library and the code is attached as supplementary material \cite{felten_toolkit_2023}.

Through a Bayesian hyperparameter search of 100 runs, we tuned the batch size, learning rate, ER buffer size, number of layers, and hidden dimension across all reported models, environments, and objective dimensions (details in the supplementary material). We compare LCN with two state-of-the-art multi-policy baselines: PCN \cite{reymond_pareto_2022} and GPI-LS \cite{alegre_gpi}, on widely used MO evaluation metrics.  To fairly compare them, we trained all algorithms for a maximum of $30,000$ steps.

\subsection{Evaluation Metrics}
\label{subsec:eval_metrics}
We use three common Multi-Objective (MO) evaluation metrics to compare our method with the baselines: one axiomatic metric (hypervolume) and two utility-based metrics (expected utility and Sen welfare). Axiomatic metrics evaluate the quality of a solution set by assuming that the true Pareto front represents the optimal solution. They do so without requiring any knowledge of the decision-maker’s preferences over objectives, instead focusing purely on the Pareto-optimality and diversity of the solutions \cite{hayes_practical_2022}. In contrast, utility-based metrics assume that the decision-maker has preferences over the objectives. These metrics either encode a specific utility function or assume a distribution over (or class of \cite{zintgraf2015quality}) possible utility functions to evaluate the solutions \cite{hayes_practical_2022}.

\textbf{Hypervolume}: an axiomatic metric that measures the volume of a set of points relative to a specific reference point and is maximized for the Pareto front. In general, it assesses the quality of a set of non-dominated solutions, its diversity, and spread \cite{hayes_practical_2022}. It's defined as:

\begin{equation}
    \texttt{HV}(CS, \vv_{ref}) = \texttt{Volume} \left(\bigcup_{\pi \in CS}\left[\vv_{ref}, \vv^\pi \right] \right),
\end{equation}
where $CS$ is the set of non-dominated policies, $\texttt{Volume}(\cdot)$ computes the Lebesgue measure of the input space, and $[\vv_{ref}, \vv^\pi]$ is the box spanned by the reference and policy value.

\textbf{Expected Utility Metric (EUM)}: a utility-based metric that measures the expected utility of a given set of solutions, under some assumed distribution of utility functions. Since the true utility function of the decision-maker is unknown, we sample utility functions from a prior distribution and evaluate, for each sampled utility, the maximal score within the solutions set. The EUM score is then computed as the average utility across all sampled decision-makers. Specifically, in this paper, we assume that utilities can be expressed as linear combinations of the objective values, i.e., $ u(\vv^\pi) = w^\top \vv^\pi$, where $w$ is a non-negative weight vector representing a particular preference trade-off among objectives. For each weight vector $w$:  $w_i \ge 0$ and $\sum_i w_i = 1$, a standard weighted-sum multi-objective utility function. We sample a set of 50 well-spaced weight vectors via the Riesz s-Energy method, which ensures diversity by spreading the samples as evenly as possible over the space of valid utility weights (see \cite{ref_dirs_energy} for details). EUM measures the actual expected utility of the policies, and is more interpretable compared to axiomatic approaches like the hypervolume \cite{hayes_practical_2022}. The score is defined as:
\begin{equation}
    \texttt{EUM}(CS) = \mathbb{E}_{P_u} \left[ \max_{\pi \in CS} u(\vv^\pi) \right],
\end{equation}
where $CS$ is the set of non-dominated policies,  $P_u$ is a distribution over utility functions (we use 100 equidistant weight vectors as was done in \cite{alegre_gpi}) and $\max_{\pi \in CS} u(v^\pi)$ is the value of the best policy in the $CS$, according to the utility function $u$, defined by the sampled weights.

\textbf{Sen Welfare}: this metric is based on a welfare function, inspired by Amartya Sen's social welfare theory. It integrates total efficiency and equality into a unified metric. Equality is quantified using the Gini coefficient, a widely used statistical measure of inequality originally developed to assess income or wealth distribution within a population. It is derived from the Lorenz curve, which plots the cumulative proportion of total reward (or resource) received by the bottom fraction of groups, ordered from the least to the most rewarded.

The Gini coefficient is defined as the area between the Lorenz curve and the line of perfect equality (a 45-degree line), normalized by the total area under the line of perfect equality. It ranges from 0 to 1, where 0 indicates perfect equality (all groups receive equal rewards), and 1 indicates maximal inequality (all rewards are concentrated in a single group) \cite{sen_poverty_1976}.

Formally, the Sen Welfare score for a policy $\pi$ is computed as:

\begin{equation}
    \texttt{SW}(\pi) = \bigg(\sum_{i}\vv^\pi_i\bigg) (1 - \texttt{GI}(\vv^\pi)),
\end{equation}
where $\sum_{i}\boldsymbol{\vv^\pi}_i$ is the sum of the returns of all objectives in policy $\pi$, and $\texttt{GI}(\vv^\pi)$ is the Gini coefficient of the return vector $\vv^\pi$. We use this metric for comparative purposes, reflecting a balanced scenario where both efficiency and equality are considered. Sen Welfare has been utilized in economic simulations employing RL before \cite{zheng_ai_2022}. A higher Sen Welfare value signifies increased efficiency and equity.

\subsection{Results}
We use the \textit{Deep Sea Treasure} environment as a sanity check to verify that our methods can effectively optimize in a simple, small-scale setting. In \Cref{tbl:dst_results}, we show that LCN performs on par with PCN on the \textit{Deep Sea Treasure} environment. We next focus on the large-scale MO-TNDP environment. The results discussed here are based on five independently seeded runs. \Cref{fig:boxplots} (a) presents a comparison between PCN, GPI-LS and LCN across all objectives in the MO-TNDP-Xi'an environment, and \Cref{fig:boxplots} (b) compares the vanilla LCN to the reference point alternatives.

\begin{figure}%
    \centering
    \subfloat[\centering]{{\includegraphics[width=7cm]{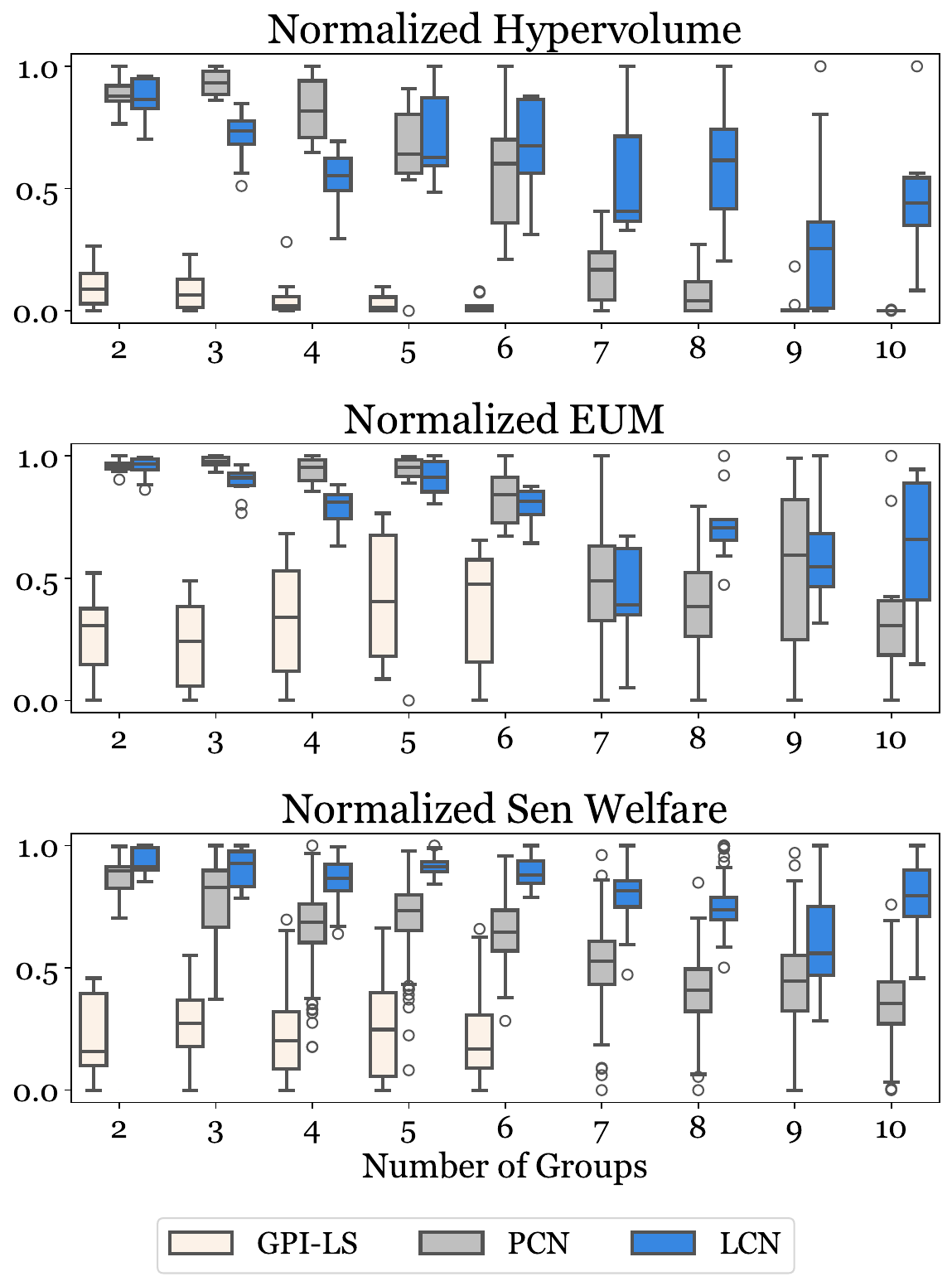} }}%
    \qquad
    \subfloat[\centering]{{\includegraphics[width=7cm]{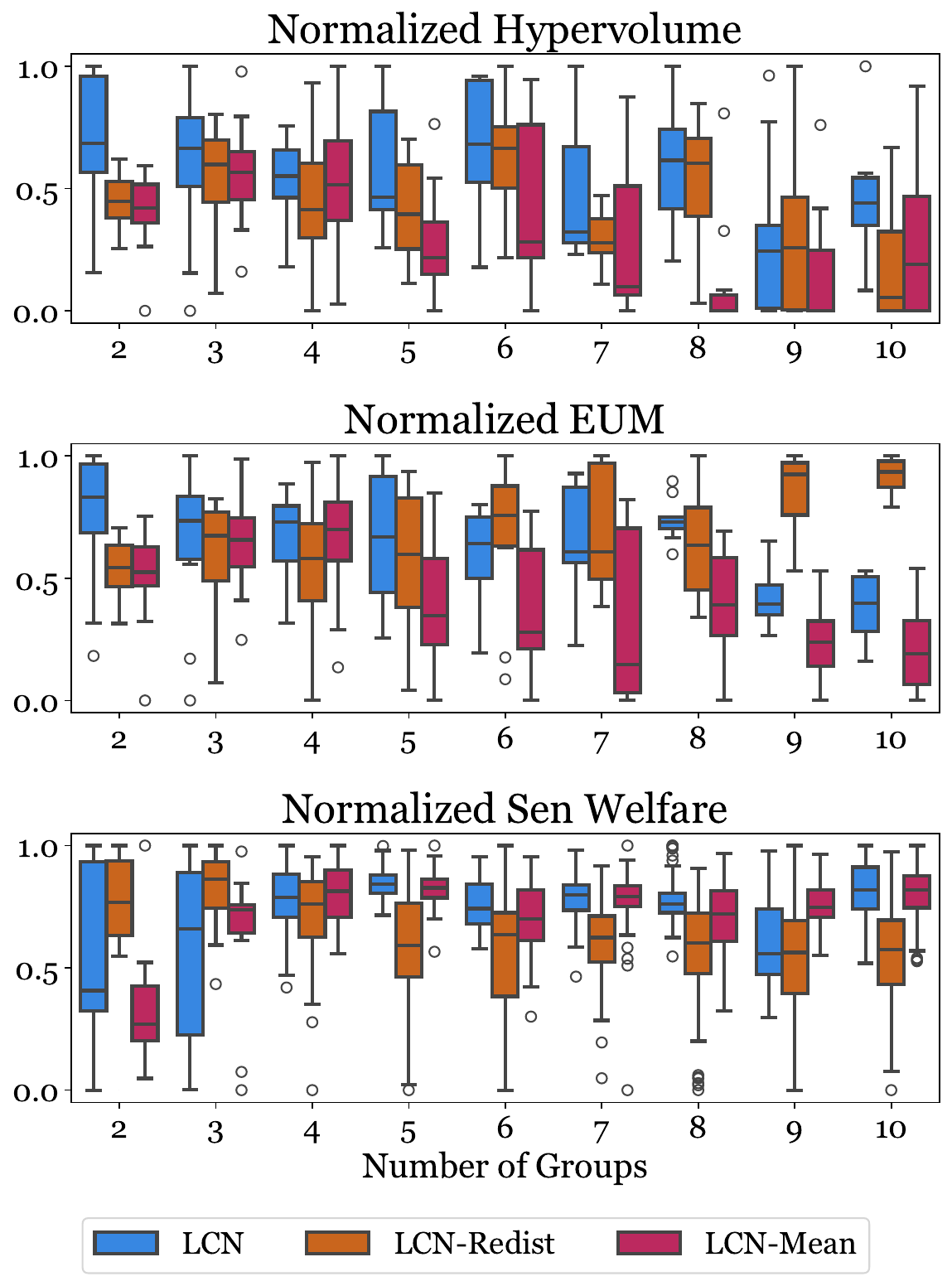} }}%
    \Description{Two-panel figure: (a) Performance comparison of LCN, PCN, and GPI-LS across all objectives in the Sen Welfare measure for Xi'an, showing that LCN outperforms PCN and GPI-LS, and also scales better in hypervolume and EUM as the number of objectives increases. (b) Comparison of trained policies for the proposed LCN, LCN-Redist, and LCN-Mean models.}
    \caption{(a) LCN outperforms PCN and GPI-LS across all objectives in the Sen Welfare measure (Xi'an). Additionally, LCN outperforms PCN in hypervolume when the number of objectives $> 4$ and in EUM for objectives $> 6$, showcasing its scalability over the objective space. (b) A comparison of the trained policies of the proposed LCN, LCN-Redist and LCN-Mean models.}%
    \label{fig:boxplots}%
\end{figure}

\begin{table}[]
\caption{Results on the \textit{Deep Sea Treasure} environment.}
\centering
\begin{tabular}{llll}
 & \multicolumn{1}{c}{\textbf{GPI-LS}} & \multicolumn{1}{c}{\textbf{PCN}} & \multicolumn{1}{c}{\textbf{LCN}} \\ \hline
\multicolumn{1}{|l|}{\texttt{HV}} & \multicolumn{1}{|l|}{$22622.8 \pm 55.4$} & $22845.4 \pm 9.6$ & \multicolumn{1}{l|}{$22838.0 \pm 0.0$} \\ \hline
\multicolumn{1}{|l|}{\texttt{EUM}} & \multicolumn{1}{|l|}{$53.86 \pm 0.02$} & $53.82 \pm 0.02$ & \multicolumn{1}{l|}{$53.76 \pm 0.03$} \\ \hline
\end{tabular}
\label{tbl:dst_results}
\end{table}

\subsection{LCN Outperforms PCN on Many-Objective Settings}
As shown in \Cref{fig:boxplots} (a), GPI-LS exhibits significantly lower performance across all objectives compared to PCN and LCN. (We limited GPI-LS to up to six objectives due to rapidly increasing runtime.) This performance gap is primarily due to the combination of the high-dimensional state-action-reward space and the limited number of time steps. The original experiments on larger domains were conducted for over $200$k steps; to ensure a fair comparison, we restricted training to $30$k steps.

We focus the remainder of the discussion on comparing PCN and LCN. PCN demonstrates strong performance on hypervolume and EUM metrics in environments with 2 to 6 objectives. This is expected, since PCN is designed to learn diverse, non-Pareto-dominated solutions that directly maximize these metrics. However, as the number of objectives increases, LCN begins to outperform PCN, even on hypervolume and EUM. This suggests that LCN is capable of learning a diverse set of policies that cover a broad region of the solution space, despite primarily targeting fair solutions.

Although this result may appear counterintuitive at first, it can be explained by the fact that the set of non-Pareto-dominated solutions grows rapidly with the number of objectives, making supervised learning significantly more difficult for PCN. Notably, when the number of objectives exceeds seven, PCN’s hypervolume performance collapses. In contrast, LCN continues to scale effectively, benefiting from the relatively smaller size of the non-Lorenz-dominated set.

LCN consistently outperforms PCN in Sen Welfare across all objectives. The Sen Welfare metric, which promotes solutions balancing efficiency and equality, shows that LCN excels in generating effective policies even when the solution space is constrained. In particular, LCN maintains its superior performance relative to PCN even as the number of objectives increases.

\begin{figure}[ht]
    \centerline{\includegraphics[width=\columnwidth]{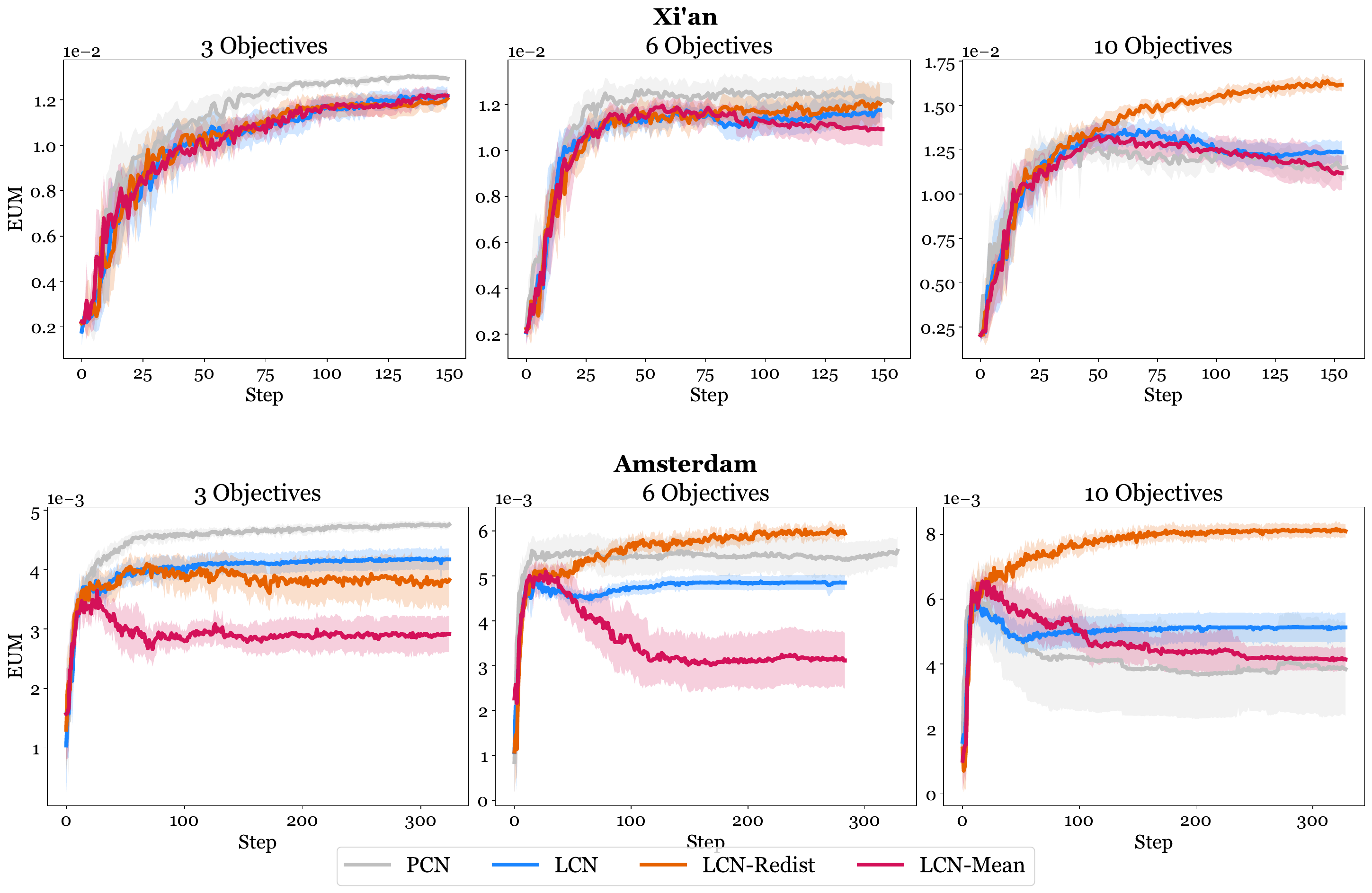}}
    \Description{Learning curves showing EUM performance for 3 and 10 objectives.}
    \caption{Learning Curves for EUM on 3 and 10 objectives (curves for all objectives are in the supplementary material).}
    \label{fig:eum_3_10}
\end{figure}

\begin{figure}[ht]
    \centerline{\includegraphics[width=\columnwidth]{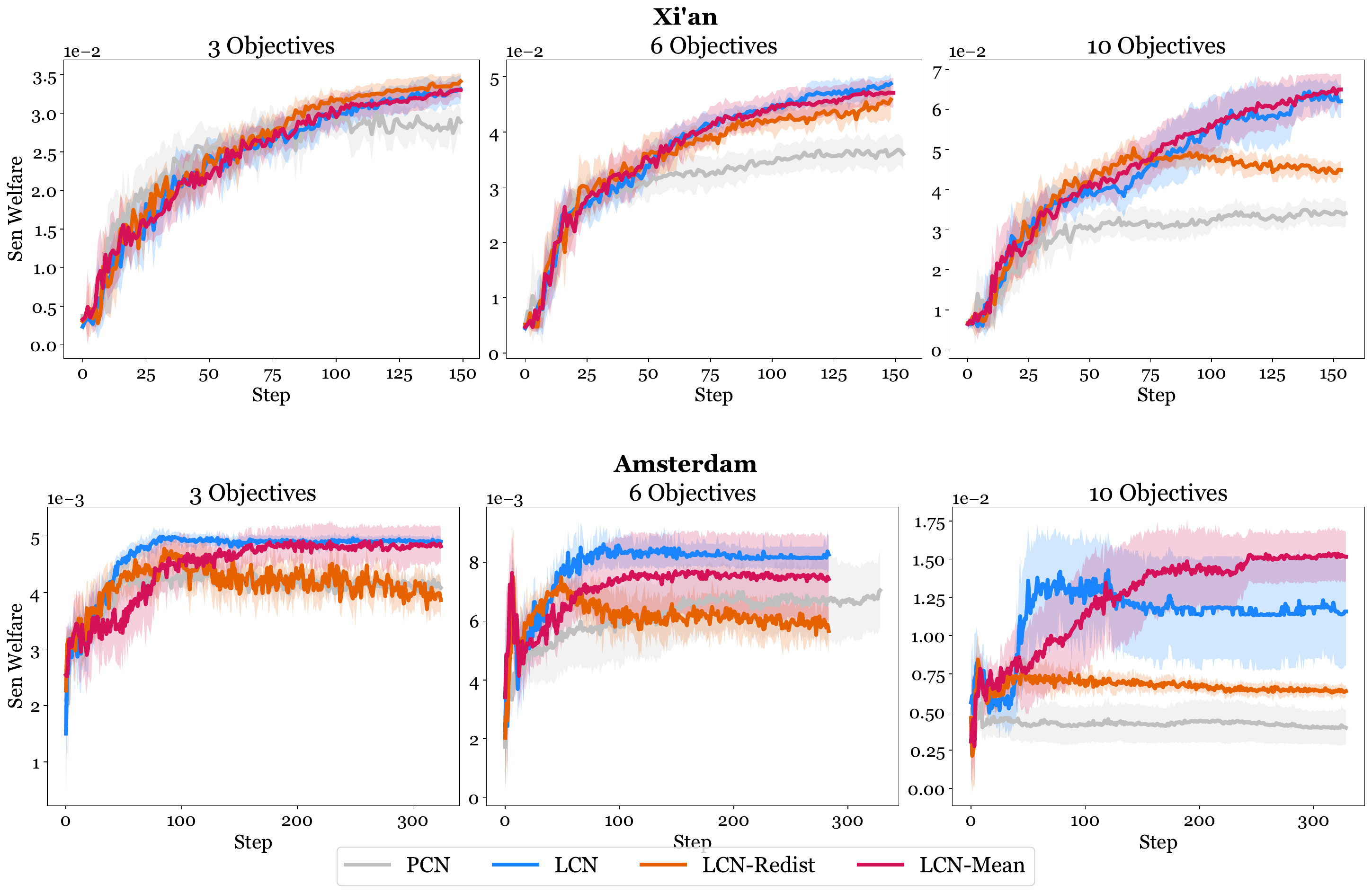}}
    \Description{Learning curves showing Sen Welfare performance for 3 and 10 objectives.}
    \caption{Learning Curves for Sen Welfare on 3 and 10 objectives (curves for all objectives are in the supplementary material).}
    \label{fig:sw_3_10}
\end{figure}

\subsection{Reference Points Improve Training}
In \Cref{fig:boxplots} (b), we compare the reference point mechanisms (LCN-Redist and LCN-Mean) with vanilla LCN. While the axiomatic hypervolume metric shows minimal impact of reference points on the model’s performance, the utility-based metrics, particularly EUM and Sen Welfare, reveal a different story. Vanilla LCN performs well when the objective space is limited. However, as the number of objectives increases, introducing reference points improves stability and outperforms using raw distances from non-dominated points.

Specifically, LCN-Redist demonstrates superior performance in EUM when $d > 5$. Although this seems counterintuitive, the redistribution mechanism creates a reference point with equal objectives. This ensures that all objectives are represented, even if they were absent in the original vector. This approach promotes policies that achieve balanced trade-offs across the solution space. The effectiveness of LCN-Redist is further illustrated by the learning curves in \Cref{fig:eum_3_10}.

Conversely, LCN-Mean performs the worst in EUM as the number of objectives increases. This outcome can be attributed to the potential skewing of the mean vector by large disparities among dimensions. If a group is consistently underrepresented in the collected experiences, its mean dimension will be close to zero, negatively affecting EUM. However, LCN-Mean performs great in Sen Welfare across all objectives, offering better stability than LCN-Redist. This result is expected, as LCN is designed to maximize Sen Welfare. LCN-Mean effectively balances outliers and creates reference points that balance efficiency and equality with minimal intervention. The learning curves for Sen Welfare are presented in \Cref{fig:sw_3_10}.

In \Cref{fig:cardinality}, we examine how the cardinality of the final, non-dominated policy sets evolves as the number of objectives increases (in the Xi'an Environment). PCN exhibits steep growth: it offers fewer than ten policies for up to three objectives, but this number rises above sixty for ten objectives. This shows, empirically, the complexity discussed in \Cref{subsec:filtering_experiences}, i.e., that as the number of objectives grows, the PCN policy set rapidly becomes large. In contrast, LCN and LCN-Mean maintain relatively stable cardinalities across all objective counts, with the trained policy set size remaining below ten throughout. This reflects the regularizing effect of reference points in constraining the solution space. LCN-Redist displays an interesting behavior: its cardinality remains low for up to six objectives; above six objectives it LCN-Redist's cardinality increases, approaching the size of PCN’s set of non-dominated policies. While the redistribution mechanism curbs the growth of the ER buffer in fewer objective settings, as dimensionality increases, the filtering retains more diverse experiences, which also explains why EUM remains high for LCN-Redist.

Overall, our experiments suggest the following guidance for decision-makers: if they care about obtaining many diverse policies that cover a wide range of fair trade-offs, they should choose LCN-Redist. However, if they prioritize strict fairness and are comfortable with a smaller set of policies, LCN is preferable for problems with fewer objectives, while LCN-Mean offers stable policies for higher-dimensional objective problems.

\begin{figure}[ht]
    \centerline{\includegraphics[width=\columnwidth/2]{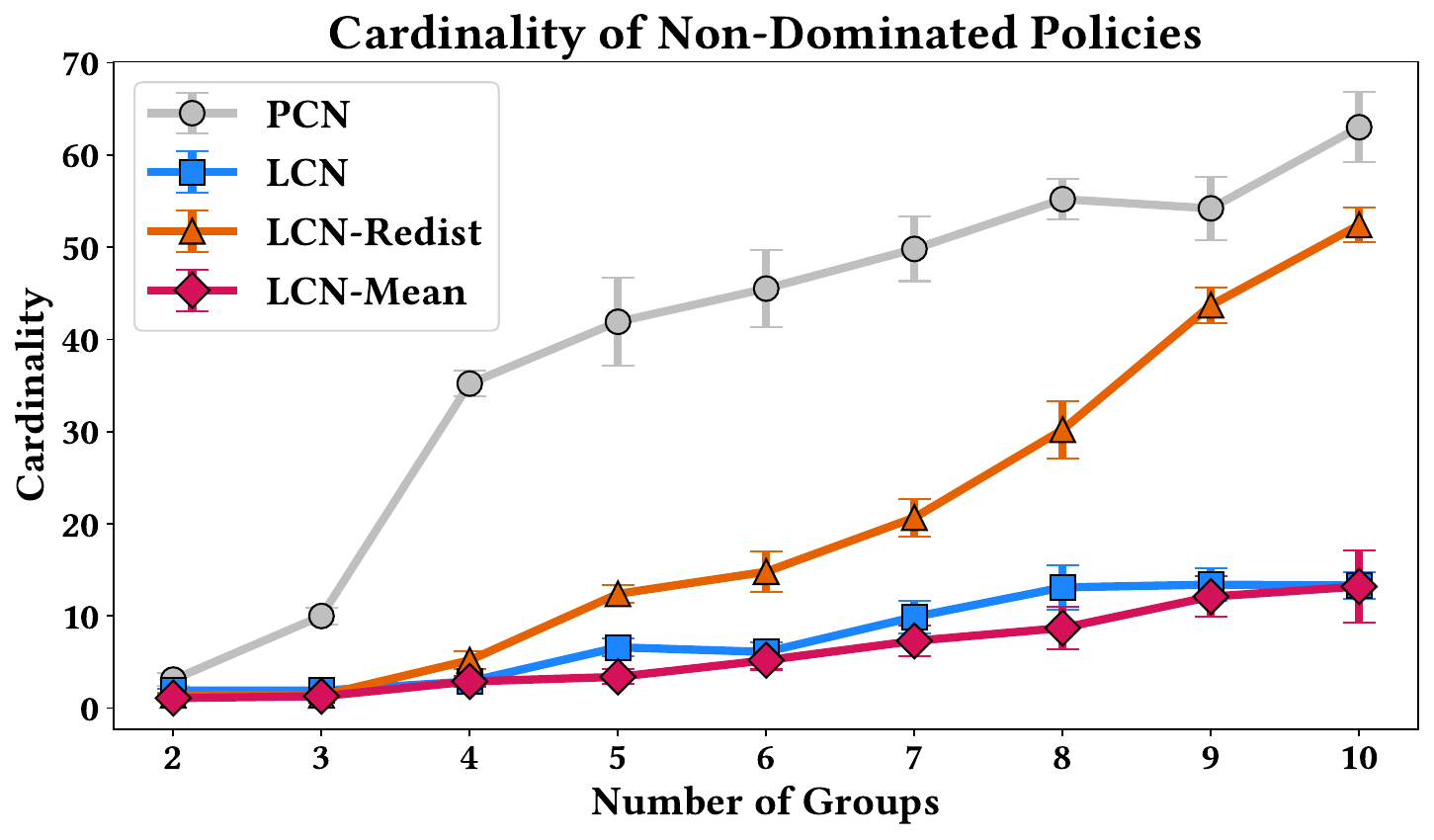}}
    \Description{A line chart showing how the size of non-dominated policy sets changes with the number of objectives. The PCN line rises steeply as objectives increase, while LCN and LCN-Mean remain relatively stable. The LCN-Redist line stays low until six objectives, then starts increasing, approaching the trend of PCN.}
    \caption{Cardinality of the non-dominated policy sets after training, as the number of objectives increases (Xi'an environment). PCN grows rapidly, while LCN and LCN-Mean remain manageable throughout. LCN-Redist stays low up to six objectives before increasing similarly to PCN.}
    \label{fig:cardinality}
\end{figure}

\begin{figure}[ht]
    \centerline{\includegraphics[width=5.2in]{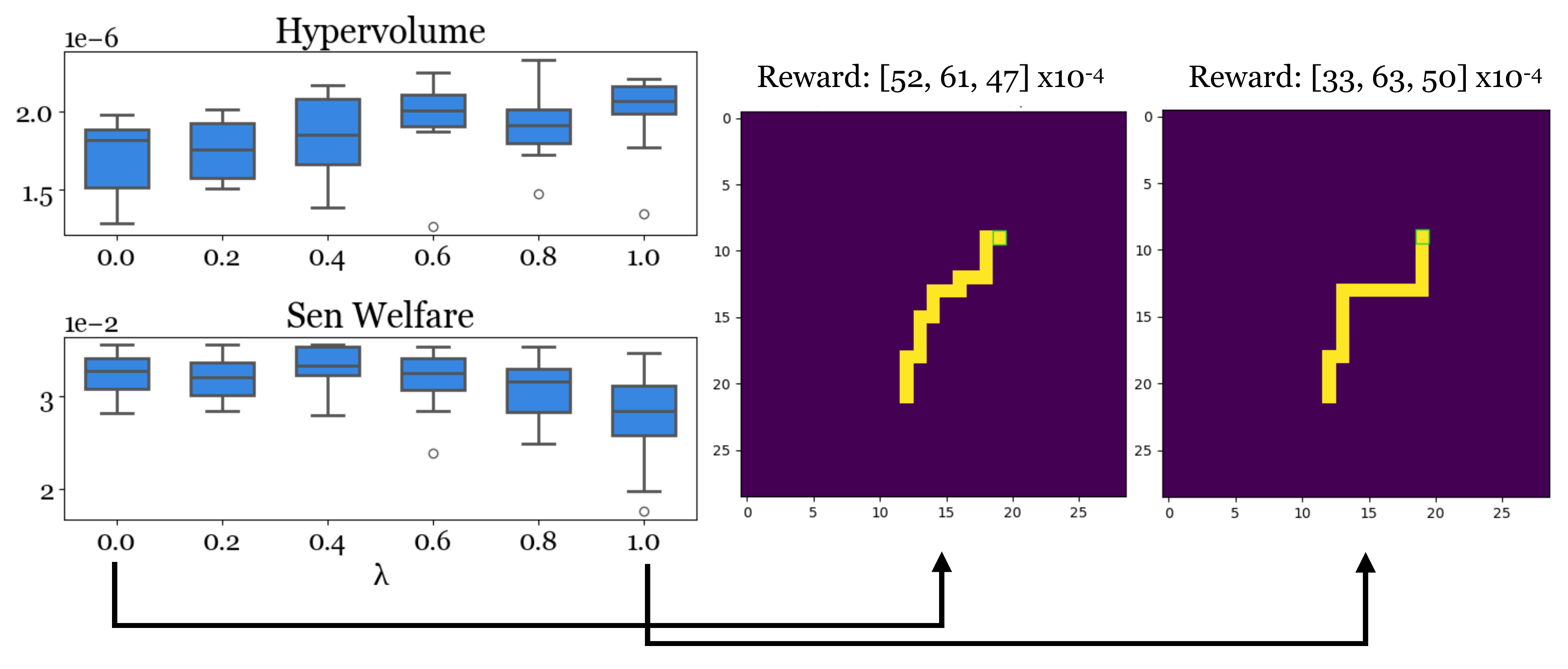}}
    \Description{A plot showing the effect of varying lambda on reward distribution. Lambda ranges from 0, emphasizing fair distribution, to 1, allowing less fair alternatives.}
    \caption{$\lambda$-LCN offers flexibility between emphasizing fair distribution of rewards ($\lambda=0$) and a relaxation that allows less fair alternatives ($\lambda=1$).}
    \label{fig:lcn_lambda_results}
\end{figure}

\begin{figure}[h]
     \centerline{\includegraphics[width=4in]{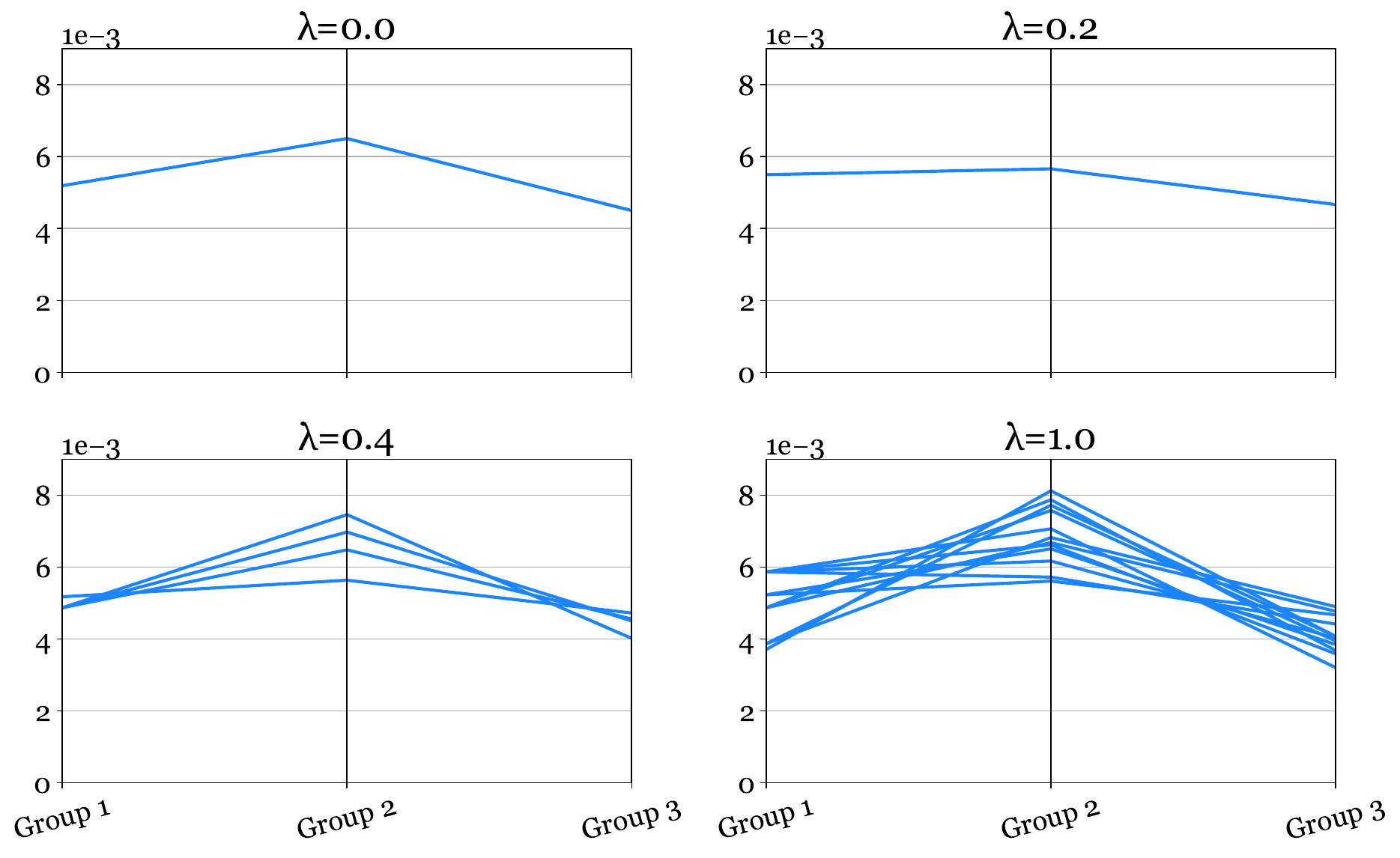}}
    \Description{Parallel coordinate plots of the approximated Pareto front for different values of lambda. As lambda increases, solutions become less constrained by fairness, resulting in a larger coverage set.}
    \caption{Parallel coordinate plots of the approximated front for varying values of $\lambda$. As $\lambda$ increases, the solutions become less constrained by fairness, resulting in a larger coverage set.}
    \label{fig:lambda_pc_xian}
\end{figure}

\subsection{\texorpdfstring{$\lambda$}{lambda}-LCN Can be Used to Achieve Control over the Degree of Fairness Preference}

In \Cref{fig:lcn_lambda_results}, we illustrate the flexibility of $\lambda$-LCN in achieving diverse solutions across different fairness preference degrees. When $\lambda \approx 0$, there is little flexibility, as the goal is to find the most equally distributed policies. This results in high performance for Sen Welfare but, as expected, lower performance in terms of hypervolume due to the concentrated solutions, which cover a smaller area. On the other hand, when $\lambda \approx 1$, the space of accepted solutions expands, leading to a larger hypervolume. However, by accepting less fair solutions, the overall Sen Welfare decreases, adding an extra layer of complexity to the decision-making process. On the right side of the figure, we also show two transport lines originating from the same cell, one for $\lambda=0$ and the other for $\lambda=1$. While the lines follow a similar direction, the placement of stations and the areas they traverse can differ substantially based on the degree of fairness preference.

In \Cref{fig:lambda_pc_xian}, we present parallel coordinate plots of the learned coverage sets for various values of $\lambda$. Here, too, we demonstrate that as $\lambda$ increases, the flexibility for distributing rewards less fairly among different groups also increases. This leads to an expanded coverage set, providing the decision maker with more policy trade-offs.

\section{Conclusion}
We addressed key challenges in multi-objective, multi-policy reinforcement learning by proposing methods that perform well over large state-action and objective spaces. We developed a new multi-objective environment for simulating Public Transport Network Design, thereby enhancing the applicability of MORL to real-world scenarios. We proposed LCN, an adaptation of state-of-the-art methods that outperforms baselines in high-reward dimensions. Finally, we present an effective method for controlling the fairness constraint. These contributions move the research field toward more realistic and applicable solutions in real-world contexts, thereby advancing the state-of-the-art in algorithmic fairness in sequential decision-making and MORL.

\begin{acks}
This research was in part supported by the European Union’s Horizon Europe research and innovation program under grant agreement No 101120406 (PEER). DM is supported by the Innovation Center for AI (ICAI, The Netherlands). WR is supported by the Research Foundation – Flanders (FWO), grant number 1197622N. F. P. Santos acknowledges funding by the European Union (ERC, RE-LINK, 101116987).
\end{acks}

\printbibliography

\appendix
\section{Theoretical Results}
\label{sec:app_proofs}
In this section, we provide the detailed theoretical results for $\lambda$-LCN that we sketched in the main paper. In particular, we show that $\lambda$-Lorenz dominance is a generalization of Lorenz dominance that can be used to flexibly set a desired fairness level. Moreover, by decreasing $\lambda$, the resulting solution set moves monotonically closer to the Lorenz front. 

We first introduce a necessary auxiliary result in \cref{lemma:lld-implies-ld} which shows that $\lambda$-Lorenz dominance implies Lorenz dominance.

\begin{lemma}
\label{lemma:lld-implies-ld}
$\forall \lambda \in [0, 1]$ and $\forall \vv, \vv' \in \mathbb{R}^d$,
\begin{equation}
    \vv \lld \vv' \implies \vv \ld \vv'. 
\end{equation}
\end{lemma}

\begin{proof}
By contradiction, assume there is some $\vv, \vv'$ and $\lambda$ such that $\vv \lld \vv'$ but $\vv$ does not Lorenz dominate $\vv'$. Then there is some smallest index $k$ such that $L(\vv)_k < L(\vv')_k$. Since $\sigma(\vv)_1 = L(\vv)_1$ and $\vv \lld \vv'$, we know that $\sigma(\vv)_1 \geq \sigma(\vv')_1$. Then for all indices $i \in \{1, \dotsc, k-1\}$ we have $\sum_{j=1}^i \sigma(\vv)_j \geq \sigma(\vv')_j$. Given that $\sum_{j=1}^k \sigma(\vv)_j < \sum_{j=0}^k \sigma(\vv')_j$, we have that $\sigma(\vv)_k < \sigma(\vv')_k$. However, this implies that $\forall \lambda \in [0, 1], \lambda \sigma(\vv)_k + (1- \lambda) L(\vv)_k < \sigma(\vv')_k + (1- \lambda) L(\vv')_k$ and therefore $\vv$ does not $\lambda$-Lorenz dominate $\vv'$ leading to a contradiction.
\end{proof}

In \cref{th:lld1-implies-lld2} we use \cref{lemma:lld-implies-ld} to demonstrate that when some vector $\lambda$-Lorenz dominates, it necessarily also dominates it for any smaller $\lambda$. 

\begin{theorem}
\label{th:lld1-implies-lld2}
$\forall \lambda_1, \lambda_2: 0 \leq \lambda_1 \leq \lambda_2 \leq 1$ and $\forall \vv, \vv' \in \mathbb{R}^d$,
\begin{equation}
    \vv \lld[2] \vv' \implies \vv \lld[1] \vv'. 
\end{equation}
\end{theorem}

\begin{proof}
Let $\vv \lld[2] \vv'$. From Definition 3, this can equivalently be written as $\lambda_2 \sigma(\vv) + (1-\lambda_2) L(\vv) \succ_P \lambda_2 \sigma(\vv') + (1-\lambda_2) L(\vv')$. In addition, from \cref{lemma:lld-implies-ld} we know that since $\vv \lld[2] \vv' \implies \vv \ld \vv'$. Then, for any index $i \in \{1, \dotsc, d\}$ we have that,  

\begin{align}
\lambda_2 \sigma(\vv)_i + (1 - \lambda_2) L(\vv)_i & \geq \lambda_2 \sigma(\vv')_i + (1 - \lambda_2) L(\vv')_i \\
\frac{\lambda_1}{\lambda_2} \lambda_2 \sigma(\vv)_i + \frac{\lambda_1}{\lambda_2}(1 - \lambda_2) L(\vv)_i & \geq \frac{\lambda_1}{\lambda_2} \lambda_2 \sigma(\vv')_i + \frac{\lambda_1}{\lambda_2} (1 - \lambda_2) L(\vv')_i \\
\lambda_1 \sigma(\vv)_i + (\frac{\lambda_1}{\lambda_2} - \lambda_1) L(\vv)_i &\geq \lambda_1 \sigma(\vv')_i  + (\frac{\lambda_1}{\lambda_2} - \lambda_1) L(\vv')_i \\
\lambda_1 \sigma(\vv)_i + (1 - \lambda_1) L(\vv)_i &\geq \lambda_1 \sigma(\vv')_i  + (1 - \lambda_1) L(\vv')_i
\end{align}
where the last step holds since $\lambda_1 \leq \frac{\lambda_1}{\lambda_2} \leq 1$ and $L(\vv)_i \geq L(\vv')_i$ by \cref{lemma:lld-implies-ld}.
\end{proof}

We now contribute an additional auxiliary result which guarantees that when some vector Pareto dominates another, it also $\lambda$-Lorenz dominates the vector. This result is an extension of the fact that Pareto dominance implies Lorenz dominance.

\begin{lemma}
\label{lemma:pd-implies-lld}
$\forall \lambda \in [0, 1]$ and $\forall \vv, \vv' \in \mathbb{R}^d$,
\begin{equation}
    \vv \pd \vv' \implies \vv \lld \vv'. 
\end{equation}
\end{lemma}
\begin{proof}
From Definition 3, $\lambda$-Lorenz dominance can be written as $\lambda \sigma(\vv) + (1-\lambda) L(\vv) \succ_P \lambda \sigma(\vv') + (1-\lambda) L(\vv')$. Let us first recall from Theorem 1 of \cite{perny_approximation_2013} that $\vv \pd \vv' \implies \vv \ld \vv'$. It is then necessary to demonstrate an analogous result for  $\vv \pd \vv' \implies \sigma(\vv) \pd \sigma(\vv')$. We prove this by induction.

Let $\vv = (\vv_1, \vv_2)$ and $\vv' = (\vv'_1, \vv'_2)$. By contradiction, assume that $\vv \pd \vv'$ but $\sigma(\vv)$ does not Pareto dominate $\sigma(\vv')$. This implies that there is some index $i$ such that $\sigma(\vv)_{i} < \sigma(\vv')_{i}$. Let us consider the four cases for $\sigma(\vv)$ and $\sigma(\vv')$.

If $\sigma(\vv) = \vv$ and $\sigma(\vv') = \vv'$ this cannot occur. Moreover, if both vectors are in reverse it also cannot be the case. 

When $\sigma(\vv) = (\vv_1, \vv_2)$ and $\sigma(\vv') = (\vv'_2, \vv'_1)$, $\vv'_2$ cannot be greater than $\vv_1$ since by transitivy then also $\vv'_1 > \vv_1$ which is a contradiction. Furthermore, $\vv'_1$ cannot be greater than $\vv_2$ because then $\vv'_1 > \vv_1$ which is again a contradiction. 

The final case, where $\sigma(\vv) = (\vv_2, \vv_1)$ and $\sigma(\vv') = (\vv'_1, \vv'_2)$, gives the same contradictions.

Assuming that the result holds for vectors of dimension $d$, we demonstrate that it must also hold for $d+1$. Let $\vv, \vv' \in \mathbb{R}^d$ and $\vv \pd \vv'$. Consider now $(\vv, a)$ and $(\vv', b)$ which are extensions of the vectors and $(\vv, a) \pd (\vv', b)$. Then $a > b$. By contradiction, assume again that $\sigma(\vv, a)$ does not dominate $\sigma(\vv', b)$. Let $i$ be the smallest index where $\sigma(\vv, a)_i < \sigma(\vv', b)_i$. There are four cases where this may occur. Either this happens when $a$ and $b$ have both not been inserted yet, $a$ has been inserted but not $b$, $b$ has been inserted but not $a$ and both have been inserted. Clearly, when neither or both were inserted, this leads to a contradiction.

If only $b$ was inserted, by transitivity we have that $\sigma(\vv, a)_i < \sigma(\vv', b)_{i+1}$ and therefore $\sigma(\vv)_i < \sigma(\vv')_i$ which is a contradiction. Lastly, if only $a$ was inserted then $\sigma(\vv, a)_i < \sigma(\vv', b)_i$ and $\sigma(\vv', b)_i < b$ implying that $a < b$ and leading to a contradiction.
\end{proof}

Finally, we provide a proof for Theorem 1 of the main paper. This result demonstrates that when $\lambda$ is $1$, the solution set starts closest to the Pareto front and decreasing $\lambda$ to $0$ monotonically reduces the solution set until it results in the Lorenz front. As such, by selecting a $\lambda$, a decision-maker can determine their preferred balance between Pareto optimality and Lorenz fairness. We first restate the theorem below and subsequently provide the proof.

\begin{theorem} (Referred to as Theorem 1 in the main text)
\label{th:front-relations-app}
$\forall \lambda_1, \lambda_2: 0 \leq \lambda_1 \leq \lambda_2 \leq 1$ and $\forall D \subset \mathbb{R}^d$ the following relations hold.
\begin{equation}
    \lf{D} \subseteq \llf{D}{\lambda_1} \subseteq \llf{D}{\lambda_2} \subseteq \pf{D}.
\end{equation}
\end{theorem}

\begin{proof}
From \cref{lemma:pd-implies-lld} we are guaranteed that $\vv \succ_P \vv'$ implies $\vv \lld \vv'$ and therefore $\forall \lambda \in [0, 1], \llf{D}{\lambda} \subseteq \pf{D}$. In addition, \cref{th:lld1-implies-lld2} guarantees that $\llf{D}{\lambda_1} \subseteq \llf{D}{\lambda_2}$. Finally, given that $\lf{D} = \llf{D}{1}$ we have that $\forall \lambda \in [0, 1], \lf{D} \subseteq \llf{D}{\lambda}$.
\end{proof}

\section{Preparing the Xi'an and Amsterdam Environments}
\label{sec:app_env_prep}
The MO-TNDP environment released with this paper is adaptable for training an agent in any city, provided there are three elements: grid size (defined by the number of rows and columns), OD matrix, and group membership assigned to each grid cell. The grid size is specified as an argument in the constructor of the environment object, along with the file paths leading to the CSV files containing the OD matrix and group membership data. We have configured the environments for both Xi'an and Amsterdam, and these are included alongside the code for the environment.

\subsubsection*{Xi'an environment preparation}
We generated the Xi'an environment utilizing the data provided in \cite{wei_city_2020}. \footnote{source: https://github.com/weiyu123112/City-Metro-Network-Expansion-with-RL}. 
The city is divided into a grid of dimensions $H^{29\times29}$, with cells of equal size ($1km^2$). 
The OD demand matrix was formulated using GPS data gathered from 25 million mobile phones, with their movements tracked over a one-month period
Additionally, each cell is assigned an average house price index, which is categorized into quintiles. \Cref{fig:xian_groups_full} provides a comprehensive breakdown of the city into various sized groups.

\subsubsection*{Amsterdam environment preparation}
We generate and release the data associated with the Amsterdam environment.
The city is divided into a grid of dimensions $H^{35\times 47}$, consisting of equally sizes cells of $0.5km^2$.
The choice of this cell size takes into consideration Amsterdam's smaller size compared to Xi'an.
Since GPS data is unavailable for Amsterdam, we estimate the OD demand using the recently published universal law of human mobility, which states that the total mobility flow between two areas, denoted as $i$ and $j$, depends on their distance and visitation frequency \cite{schlapfer_universal_2021}. The estimation is computed using the formula:
\begin{equation}
    OD_{ij} = \mu_j \mathsf{K_i} / d^{2}_{ij} \ln(f_{max}/f_{min})
\end{equation}
Where $\mathsf{K_i}$ represents the total area of the origin location $i$, $d^{2}_{ij}$ is the (Manhattan) distance between $i,j$ and $\mu_j$ is the magnitude of flows, calculated as follows:
\begin{equation}
    \mu_j \approx \rho _{pop}(j)rad^2_jf_{max}
\end{equation}
Where $rad^2_j$ is the radius of area j. The flows are estimated for a full week, and in the model, this is accomplished by setting $f_{min}$ and $f_{max}$ to $1/7$ and $7$ respectively. Since the grid cells are of equal size in our case, the term $K$ can be omitted from the calculation. An illustration of the Amsterdam environment is presented in A detailed breakdown of the city into different-sized groups is depicted in \Cref{fig:ams_grps_full}.

Similar to the Xi'an environment, each cell in Amsterdam is associated with an average house price, sourced from the publicly available statistical bureau of the Netherlands dataset \footnote{source: https://www.cbs.nl/nl-nl/maatwerk/2019/31/kerncijfers-wijken-en-buurten-2019}.

\begin{figure}[ht]
    \centerline{\includegraphics[width=\columnwidth]{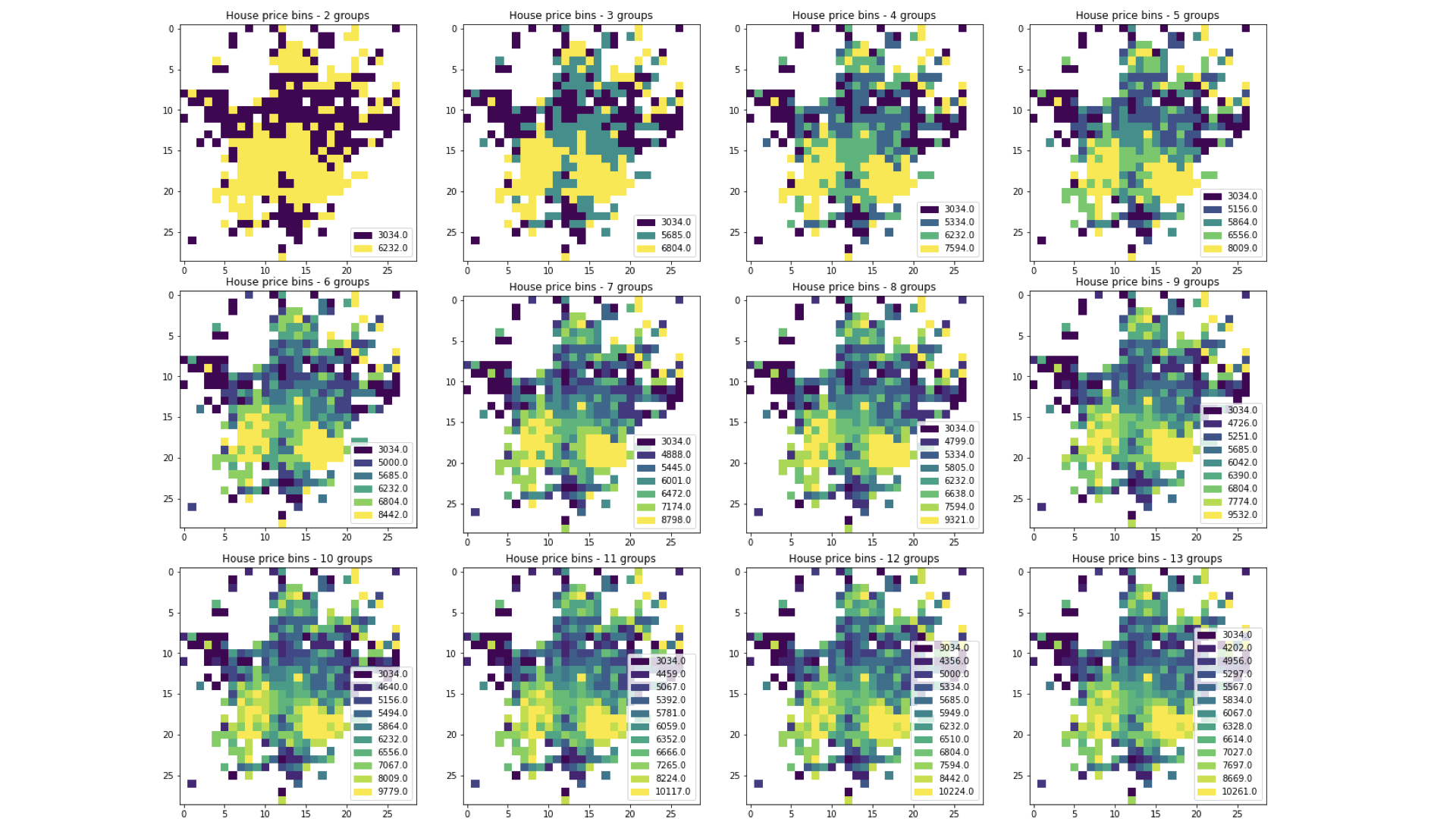}}
    \Description{Map of the MO-TNDP Xi'an environment showing different groups based on equally sized buckets of the average house price index. Each group represents a range of house prices.}
    \caption{MO-TNDP Xi'an Environment with different groups based on equally sized buckets of the average house price index.}
    \label{fig:xian_groups_full}
\end{figure}

\begin{figure}[ht]
    \centerline{\includegraphics[width=\columnwidth]{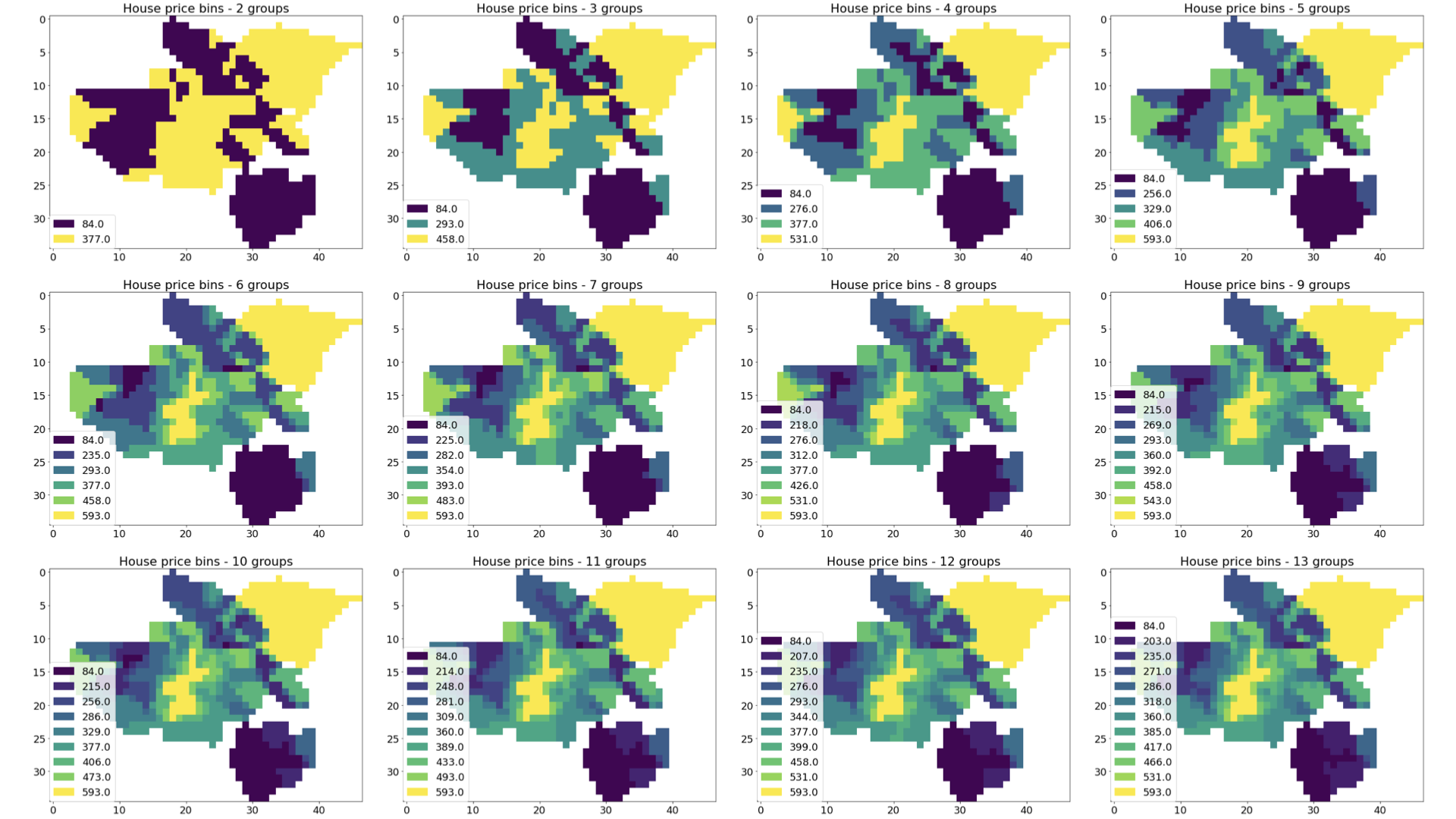}}
    \Description{Map of the MO-TNDP Amsterdam environment showing different groups based on equally sized buckets of the average house price index. Each group represents a range of house prices.}
    \caption{MO-TNDP Amsterdam Environment with different groups based on equally sized buckets of the average house price index.}
    \label{fig:ams_grps_full}
\end{figure}

\section{Experiment Reproducibility Details}
\label{sec:app_exp_details}

Each model presented in the paper for training MO-TNDP was trained for 30000 steps.
Hyperparameters were tuned via a Bayes search over 100 settings, with the following ranges:

\textbf{PCN/LCN}
\begin{itemize}
    \item Batch size: [128, 256]
    \item Learning Rate: [0.1, 0.01]
    \item Number of Linear Layers: [1, 2]
    \item Hidden Dims: [64, 128]
    \item Experience Replay Buffer Size: [50, 100]
    \item Model Updates: [5, 10]
\end{itemize}

\textbf{GPI-LS}
\begin{itemize}
    \item Network Architecture: [64, 64, 64]
    \item Learning Rate: [0.00001, 0.0001, 0.001, 0.01]
    \item Batch Size: [16, 32, 64, 128, 256, 512]
    \item Buffer Size: [256, 512, 2048, 4096, 8192, 16384, 32768]
    \item Learning Starts: 50
    \item Target Net Update Frequency: [10, 20, 50, 100]
    \item Gradient Updates: [1, 2, 5]
\end{itemize}

\cref{fig:architecture} shows the architecture used for the policy network. In the provided code, we provide the exact commands to reproduce all of our experiments, including the environments, hyperparameters, and seeds used to generate our results. Furthermore, the details of the hyperparameters we used for each experiment are available on a public Notion page \footnote{https://aware-night-ab1.notion.site/Project-B-MO-LCN-Experiment-Tracker-b4d21ab160eb458a9cff9ab9314606a7}. Finally, we commit to sharing the output model weights upon request. 

\begin{figure}[ht]
    \centerline{\includegraphics[width=5.3in]{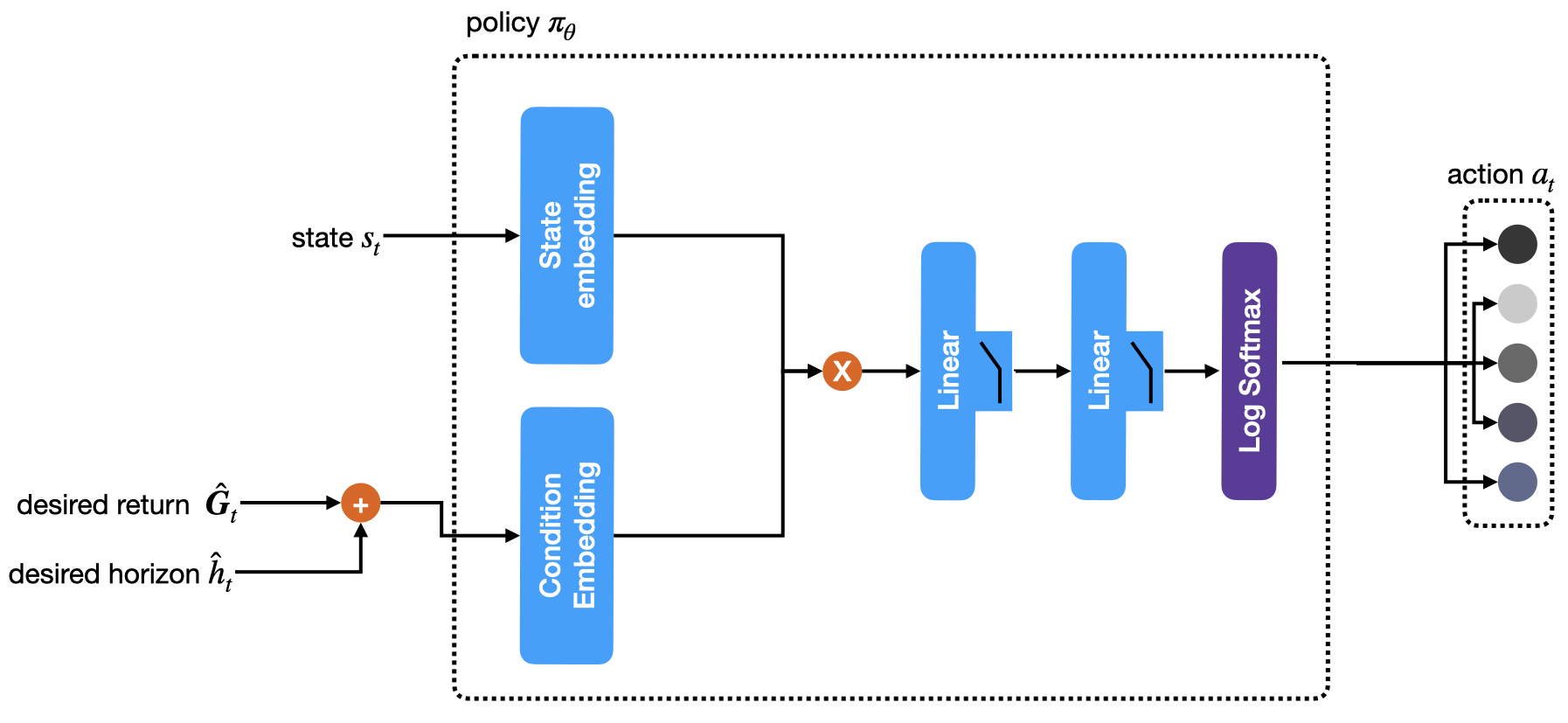}}
    \Description{Diagram showing the architecture of the policy network, including input features, the state and condition embedding, two linear layers after the embedding, and the action probability output.}
    \caption{Architecture of the policy network.}
    \label{fig:architecture}
\end{figure}

\begin{figure}[ht]
    \centerline{\includegraphics[width=\columnwidth]{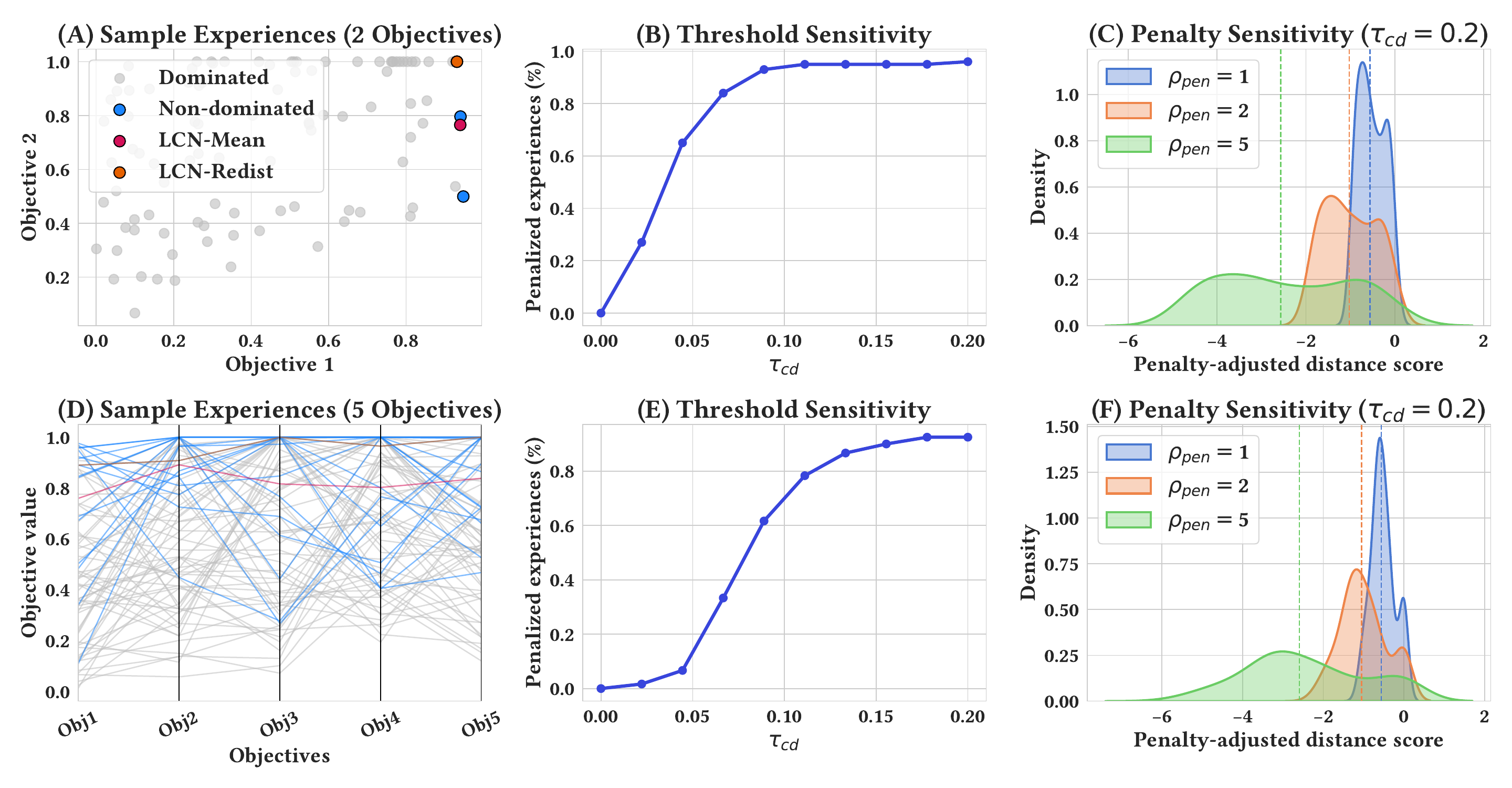}}
    \Description{.}
    \caption{Sensitivity analysis of the crowding-distance penalty: Xi'an Environment, 2 groups (A, B, C) and 5 groups (D, E, F). (A) Randomly sampled experiences, with the non-dominated set and the LCN-Mean and LCN-Redist reference points highlighted. (B) Fraction of sampled experiences that receive a crowding penalty as a function of the crowding-distance threshold ($\tau_{cd}$).(C) Kernel density estimates (KDEs) of the total penalty applied to the experiences for different penalty multipliers, assuming $\tau_{cd}=0.2$.}
    \label{fig:cd_distance_sensitivity}
\end{figure}

\section{Sensitivity Analysis of Crowding Distance Threshold and Penalty}
\label{subsec:appendix:cd_sensitivity_analysis}

We conduct a sensitivity analysis on the effects of the crowding-distance threshold $\tau_{cd}$ and the penalty multiplier $\rho_{pen}$. When the model filters experiences to maintain a consistent Experience Replay (ER) buffer, it applies a distance metric together with a crowdedness threshold. Experiences that lie too close to each other, i.e., those whose crowding distance exceeds the threshold, are penalized to promote diversity, as the objective is to train a broad set of policies for the decision-maker. We examine, in a theoretical setting, how variations in $\tau_{cd}$ and $\rho_{pen}$ influence both the proportion of penalized samples and the overall performance of the algorithm.

In \Cref{fig:cd_distance_sensitivity}, we generated random sets of experiences in a 2-dimension (Panel A) and 5-dimension setting (Panel D). We show the non-dominated points, as well as the LCN-Mean and LCN-Redist reference points. In Panels C and E we show the sensitivity of the generated experiences to the crowding distance threshold. 

The influence of the crowding distance threshold ($\tau_{cd}$) varies significantly with the dimensionality of the objective space. In low dimensions (e.g., 2D), experiences cluster closer in the crowding-distance metric. Consequently, the sensitivity to $\tau_{cd}$ is high initially, with the proportion of penalized points peaking sharply at a very low threshold (e.g., $\tau_{cd} \approx 0.1$), where nearly all points are penalized. This calls for setting a low $\tau_{cd}$ in low-dimensional settings for the crowding distance penalty to remain effective. Conversely, in higher dimensions (e.g., 5D), the crowding distance penalty peaks at a later stage (e.g., $\tau_{cd} \approx 0.2$). The points are naturally more spread out, giving the designer finer control over the number of penalized experiences.

The penalty multiplier ($\rho_{pen}$) impacts the distribution of the penalty-adjusted distance scores. With a low penalty multiplier ($\rho_{pen} = 1$), the distribution of the adjusted scores is concentrated close to zero, with low variance and thin tails, leading to a minimal scaling effect beyond the crowding distance. However, increasing $\rho_{pen}$ significantly skews the distribution, introducing heavier tails and increasing the variance of the penalty. This high multiplier aggressively penalizes points below the threshold, meaning points just below the threshold are greatly favored over those further away. A higher $\rho_{pen}$ is desirable in low-diversity environments (where policies are very close) to discriminate between them. A lower $\rho_{pen}$ suffices in high-diversity environments where distance itself is a sufficient discriminator. In this paper, we scale our environment to high dimensions, however, we operate in a relatively low-variance environment with hard constraints on the generated policies, hence we opted for a middle-point penalty of $\rho_{pen} = 2$.

\section{Additional Results}

\begin{figure*}[ht]
    \centerline{\includegraphics[width=6.5in]{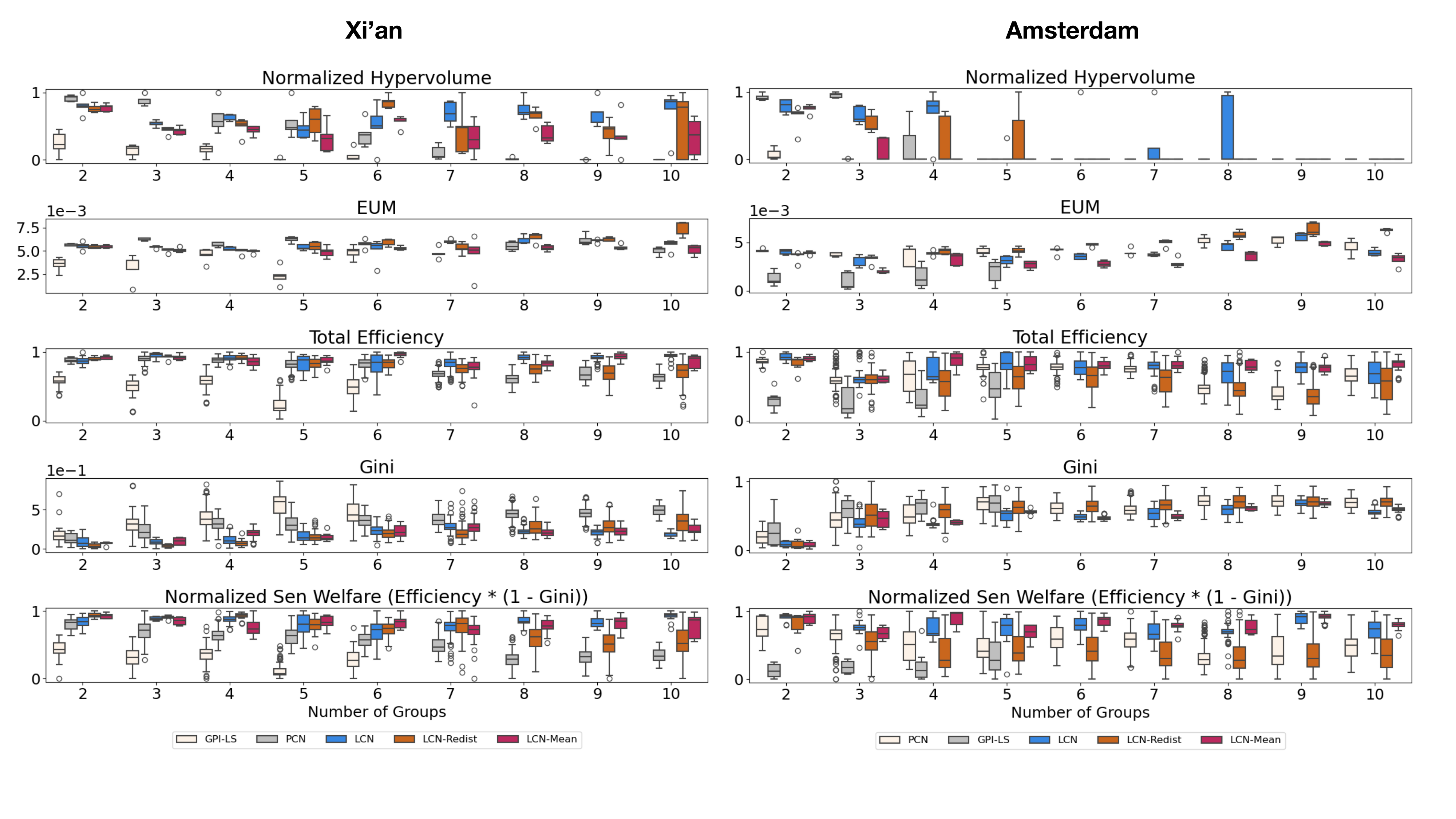}}
    \Description{Comprehensive results for the MO-TNDP Amsterdam and Xi'an environments.}
    \caption{Full results for the MO-TNDP Amsterdam and Xi'an Environments.}
    \label{fig:full_results_all_metrics}
\end{figure*}

\begin{figure*}[ht]
    \centerline{\includegraphics[width=5.5in]{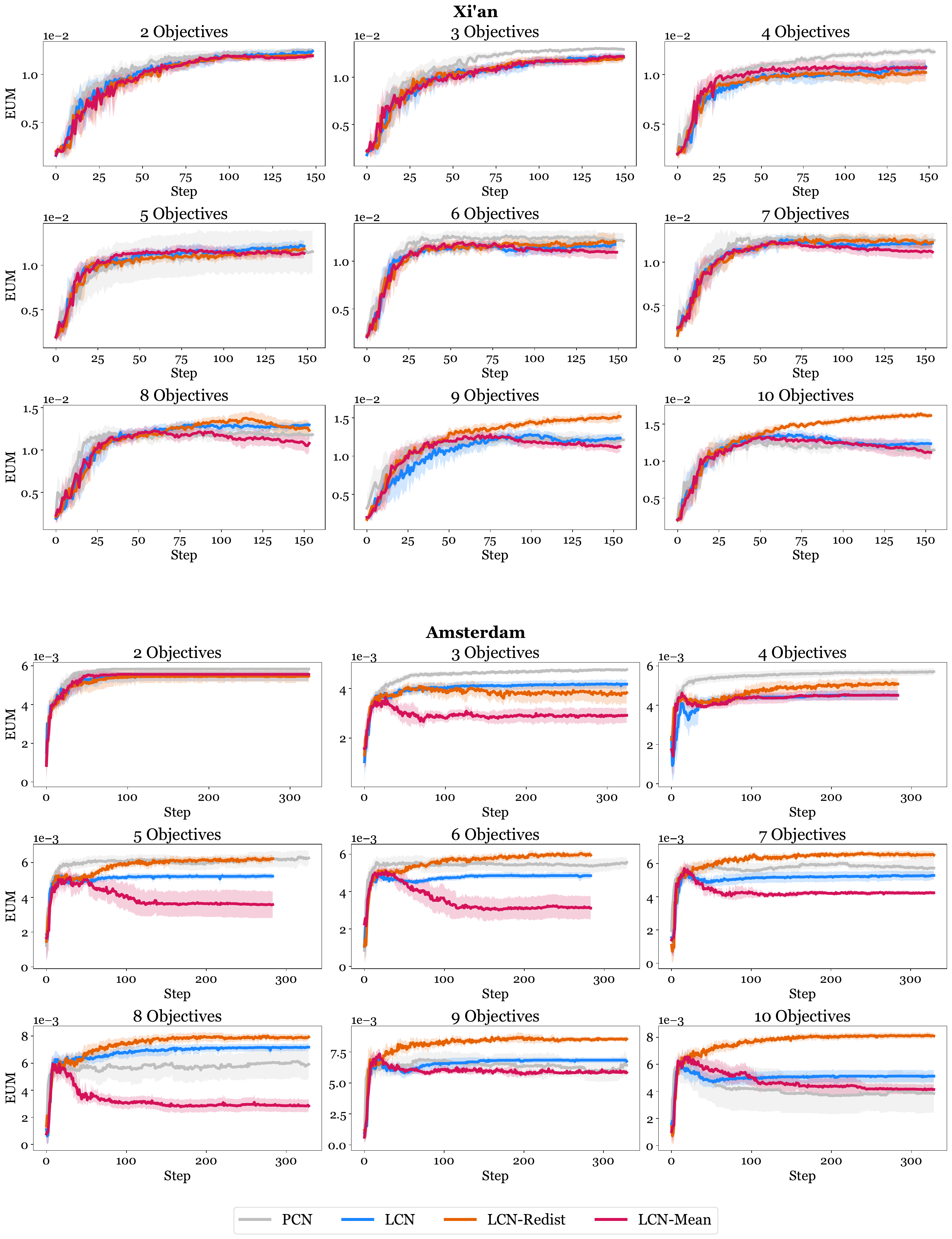}}
    \Description{Learning curves for the EUM metric across multiple objectives, showing convergence over training iterations.}
    \caption{Learning curves for EUM.}
    \label{fig:eum_lines_all}
\end{figure*}

\newpage

\begin{figure*}[ht]
    \centerline{\includegraphics[width=5.5in]{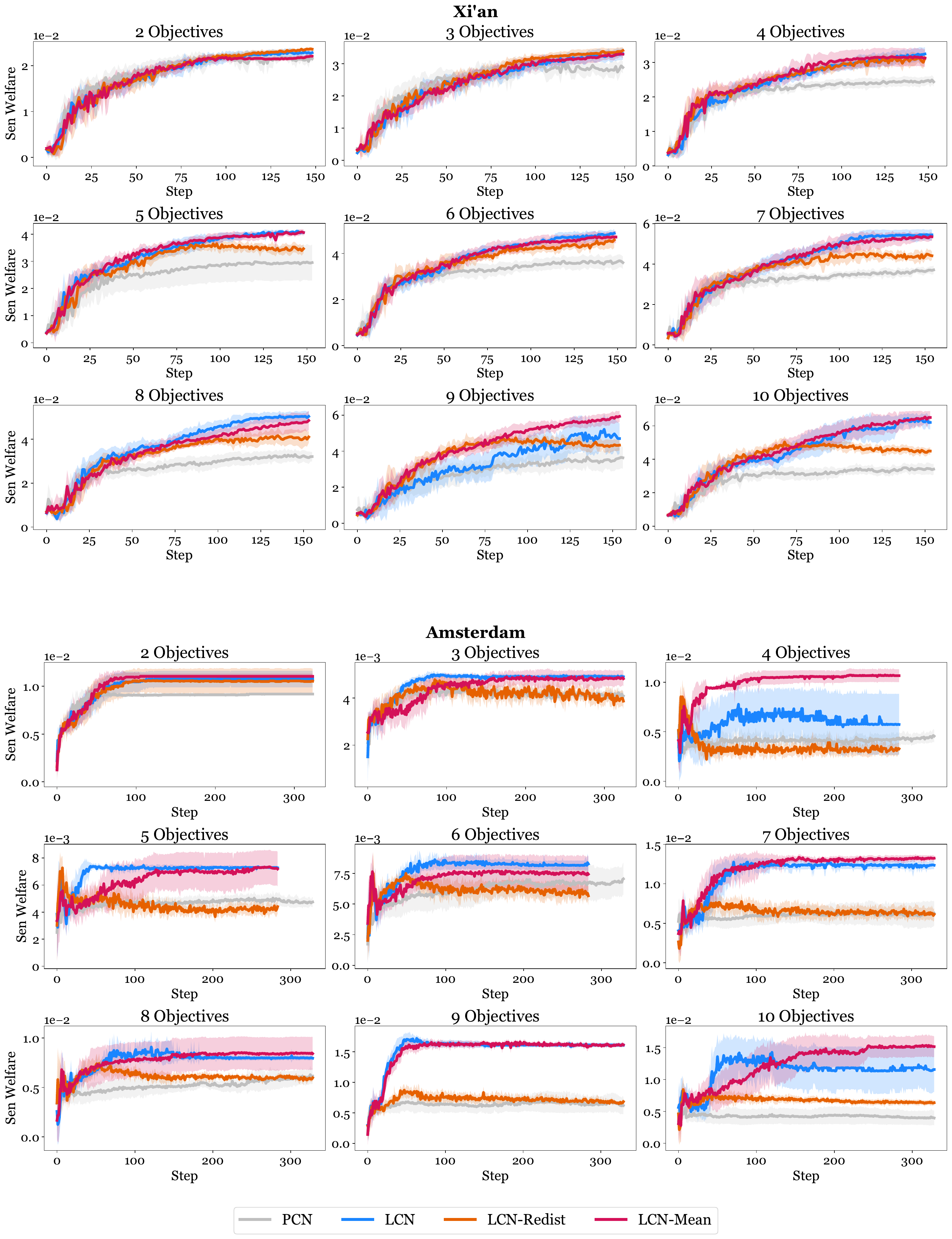}}
    \Description{Learning curves for the Sen Welfare metric across multiple objectives.}
    \caption{Learning curves for Sen Welfare.}
    \label{fig:sw_lines_all}
\end{figure*}

\begin{table*}[]
\footnotesize
\caption{Results of all models, for 1--10 objectives. Underline indicates the best results.}
\label{tbl:full_results}
\resizebox{\textwidth}{!}{
\begin{tabular}{llrrrrrrrr}
\hline
                                 & \multicolumn{9}{c}{Normalized hypervolume}                                                                                                                                                                                                                                                                              \\ \hline
                                 & \multicolumn{9}{c}{Number of Objectives}                                                                                                                                                                                                                                                                                \\ \hline
\multicolumn{1}{l|}{Xi'an}       & \multicolumn{1}{c}{2}       & \multicolumn{1}{c}{3}               & \multicolumn{1}{c}{4}               & \multicolumn{1}{c}{5}               & \multicolumn{1}{c}{6}                           & \multicolumn{1}{c}{7}       & \multicolumn{1}{c}{8}       & \multicolumn{1}{c}{9}       & \multicolumn{1}{c}{10}      \\ \hline
\multicolumn{1}{l|}{GPI-LS}      & $0.11 \pm 0.03$             & \multicolumn{1}{l}{$0.08 \pm 0.02$} & \multicolumn{1}{l}{$0.05 \pm 0.03$} & \multicolumn{1}{l}{$0.03 \pm 0.01$} & \multicolumn{1}{l}{$0.02 \pm 0.01$}             & \multicolumn{1}{l}{$--$}    & \multicolumn{1}{l}{$--$}    & \multicolumn{1}{l}{$--$}    & \multicolumn{1}{l}{$--$}    \\
\multicolumn{1}{l|}{PCN}         & $\underline{0.89 \pm 0.02}$ & $\underline{0.93 \pm 0.02}$         & $\underline{0.83 \pm 0.04}$         & $0.63 \pm 0.08$                     & $0.56 \pm 0.08$                                 & $0.17 \pm 0.04$             & $0.07 \pm 0.03$             & $0.02 \pm 0.02$             & $0.00 \pm 0.00$             \\
\multicolumn{1}{l|}{LCN}         & $0.86 \pm 0.03$             & $0.71 \pm 0.03$                     & $0.54 \pm 0.04$                     & $\underline{0.71 \pm 0.06}$         & $\underline{0.67 \pm 0.06}$                     & $\underline{0.55 \pm 0.08}$ & $\underline{0.60 \pm 0.08}$ & $0.29 \pm 0.11$             & $\underline{0.46 \pm 0.07}$ \\
\multicolumn{1}{l|}{LCN-Redist.} & $0.79 \pm 0.01$             & $0.69 \pm 0.03$                     & $0.48 \pm 0.06$                     & $0.60 \pm 0.04$                     & $0.64 \pm 0.05$                                 & $0.39 \pm 0.03$             & $0.53 \pm 0.08$             & $\underline{0.31 \pm 0.11}$ & $0.19 \pm 0.08$             \\
\multicolumn{1}{l|}{LCN-Mean}    & $0.78 \pm 0.02$             & $0.70 \pm 0.02$                     & $0.52 \pm 0.07$                     & $0.51 \pm 0.05$                     & $0.49 \pm 0.08$                                 & $0.37 \pm 0.08$             & $0.12 \pm 0.08$             & $0.15 \pm 0.08$             & $0.28 \pm 0.10$             \\ \hline
\multicolumn{1}{l|}{Amsterdam}   & \multicolumn{1}{c}{2}       & \multicolumn{1}{c}{3}               & \multicolumn{1}{c}{4}               & \multicolumn{1}{c}{5}               & \multicolumn{1}{c}{6}                           & \multicolumn{1}{c}{7}       & \multicolumn{1}{c}{8}       & \multicolumn{1}{c}{9}       & \multicolumn{1}{c}{10}      \\ \hline
\multicolumn{1}{l|}{GPI-LS}      & $0.12 \pm 0.05$             & \multicolumn{1}{l}{$0.22 \pm 0.06$} & \multicolumn{1}{l}{$0.06 \pm 0.06$} & \multicolumn{1}{l}{$0.00 \pm 0.00$} & \multicolumn{1}{l}{$--$}                        & \multicolumn{1}{l}{$--$}    & \multicolumn{1}{l}{$--$}    & \multicolumn{1}{l}{$--$}    & \multicolumn{1}{l}{$--$}    \\
\multicolumn{1}{l|}{PCN}         & $\underline{0.97 \pm 0.01}$ & $\underline{0.80 \pm 0.04}$         & $0.17 \pm 0.11$                     & $0.00 \pm 0.00$                     & $0.00 \pm 0.00$                                 & $0.00 \pm 0.00$             & $0.00 \pm 0.00$             & $0.00 \pm 0.00$             & $0.00 \pm 0.00$             \\
\multicolumn{1}{l|}{LCN}         & $0.80 \pm 0.03$             & $0.67 \pm 0.04$                     & $\underline{0.41 \pm 0.14}$         & $0.00 \pm 0.00$                     & $0.00 \pm 0.00$                                 & $0.00 \pm 0.00$             & $0.00 \pm 0.00$             & $0.00 \pm 0.00$             & $0.00 \pm 0.00$             \\
\multicolumn{1}{l|}{LCN-Redist.} & $0.79 \pm 0.03$             & $0.53 \pm 0.03$                     & $0.00 \pm 0.00$                     & $0.00 \pm 0.00$                     & $0.00 \pm 0.00$                                 & $0.00 \pm 0.00$             & $0.00 \pm 0.00$             & $0.00 \pm 0.00$             & $0.00 \pm 0.00$             \\
\multicolumn{1}{l|}{LCN-Mean}    & $0.81 \pm 0.01$             & $0.27 \pm 0.05$                     & $0.00 \pm 0.00$                     & $0.00 \pm 0.00$                     & $0.00 \pm 0.00$                                 & $0.00 \pm 0.00$             & $0.00 \pm 0.00$             & $0.00 \pm 0.00$             & $0.00 \pm 0.00$             \\ \hline
                                 & \multicolumn{9}{c}{Normalized EUM}                                                                                                                                                                                                                                                                                      \\ \hline
\multicolumn{1}{c|}{Xi'an}       & \multicolumn{1}{c}{2}       & \multicolumn{1}{c}{3}               & \multicolumn{1}{c}{4}               & \multicolumn{1}{c}{5}               & \multicolumn{1}{c}{6}                           & \multicolumn{1}{c}{7}       & \multicolumn{1}{c}{8}       & \multicolumn{1}{c}{9}       & \multicolumn{1}{c}{10}      \\ \hline
\multicolumn{1}{l|}{GPI-LS}      & $0.27 \pm 0.06$             & \multicolumn{1}{l}{$0.24 \pm 0.06$} & \multicolumn{1}{l}{$0.33 \pm 0.08$} & \multicolumn{1}{l}{$0.42 \pm 0.08$} & \multicolumn{1}{l}{$0.38 \pm 0.08$}             & \multicolumn{1}{l}{$--$}    & \multicolumn{1}{l}{$--$}    & \multicolumn{1}{l}{$--$}    & \multicolumn{1}{l}{$--$}    \\
\multicolumn{1}{l|}{PCN}         & $\underline{0.96 \pm 0.01}$ & $0\underline{.97 \pm 0.01}$         & $\underline{0.94 \pm 0.02}$         & $0.86 \pm 0.10$                     & $0.83 \pm 0.04$                                 & $\underline{0.57 \pm 0.09}$ & $0.55 \pm 0.04$             & $0.39 \pm 0.07$             & $0.27 \pm 0.05$             \\
\multicolumn{1}{l|}{LCN}         & $0.95 \pm 0.01$             & $0.89 \pm 0.02$                     & $0.79 \pm 0.02$                     & $\underline{0.91 \pm 0.02}$         & $0.80 \pm 0.02$                                 & $0.50 \pm 0.06$             & $\underline{0.74 \pm 0.03}$ & $0.42 \pm 0.04$             & $0.38 \pm 0.04$             \\
\multicolumn{1}{l|}{LCN-Redist.} & $0.92 \pm 0.01$             & $0.88 \pm 0.02$                     & $0.73 \pm 0.04$                     & $0.89 \pm 0.02$                     & $\underline{0.83 \pm 0.04}$                     & $0.52 \pm 0.07$             & $0.63 \pm 0.07$             & $\underline{0.86 \pm 0.05}$ & $\underline{0.92 \pm 0.02}$ \\
\multicolumn{1}{l|}{LCN-Mean}    & $0.91 \pm 0.01$             & $0.89 \pm 0.01$                     & $0.78 \pm 0.04$                     & $0.84 \pm 0.02$                     & $0.71 \pm 0.03$                                 & $0.26 \pm 0.09$             & $0.39 \pm 0.07$             & $0.25 \pm 0.05$             & $0.22 \pm 0.06$             \\ \hline
\multicolumn{1}{l|}{Amsterdam}   & \multicolumn{1}{c}{2}       & \multicolumn{1}{c}{3}               & \multicolumn{1}{c}{4}               & \multicolumn{1}{c}{5}               & \multicolumn{1}{c}{6}                           & \multicolumn{1}{c}{7}       & \multicolumn{1}{c}{8}       & \multicolumn{1}{c}{9}       & \multicolumn{1}{c}{10}      \\ \hline
\multicolumn{1}{l|}{GPI-LS}      & $0.30 \pm 0.07$             & \multicolumn{1}{l}{$0.85 \pm 0.05$} & \multicolumn{1}{l}{$0.92 \pm 0.02$} & \multicolumn{1}{l}{$0.67 \pm 0.07$} & \multicolumn{1}{l}{$--$}                        & \multicolumn{1}{l}{$--$}    & \multicolumn{1}{l}{$--$}    & \multicolumn{1}{l}{$--$}    & \multicolumn{1}{l}{$--$}    \\
\multicolumn{1}{l|}{PCN}         & $\underline{0.99 \pm 0.00}$ & $\underline{0.95 \pm 0.01}$         & $\underline{0.75 \pm 0.04}$         & $\underline{0.80 \pm 0.04}$         & $0.76 \pm 0.04$                                 & $0.62 \pm 0.09$             & $0.59 \pm 0.07$             & $0.36 \pm 0.10$             & $0.32 \pm 0.09$             \\
\multicolumn{1}{l|}{LCN}         & $0.93 \pm 0.02$             & $0.70 \pm 0.03$                     & $0.21 \pm 0.05$                     & $0.55 \pm 0.02$                     & $0.58 \pm 0.02$                                 & $0.47 \pm 0.04$             & $0.79 \pm 0.02$             & $0.44 \pm 0.03$             & $0.50 \pm 0.03$             \\
\multicolumn{1}{l|}{LCN-Redist.} & $0.92 \pm 0.02$             & $0.55 \pm 0.08$                     & $0.46 \pm 0.07$                     & $0.79 \pm 0.02$                     & $\underline{0.86 \pm 0.02}$                     & $\underline{0.89 \pm 0.03}$ & $\underline{0.90 \pm 0.02}$ & $\underline{0.88 \pm 0.02}$ & $\underline{0.92 \pm 0.01}$ \\
\multicolumn{1}{l|}{LCN-Mean}    & $0.94 \pm 0.01$             & $0.15 \pm 0.06$                     & $0.21 \pm 0.05$                     & $0.17 \pm 0.08$                     & $0.12 \pm 0.07$                                 & $0.11 \pm 0.02$             & $0.12 \pm 0.03$             & $0.22 \pm 0.01$             & $0.36 \pm 0.02$             \\ \hline

                                 & \multicolumn{9}{c}{Normalized Sen Welfare}                                                                                                                                                                                                                                                                                      \\ \hline
\multicolumn{1}{c|}{Xi'an}       & \multicolumn{1}{c}{2}       & \multicolumn{1}{c}{3}               & \multicolumn{1}{c}{4}               & \multicolumn{1}{c}{5}               & \multicolumn{1}{c}{6}                           & \multicolumn{1}{c}{7}       & \multicolumn{1}{c}{8}       & \multicolumn{1}{c}{9}       & \multicolumn{1}{c}{10}      \\ \hline
\multicolumn{1}{l|}{GPI-LS} & $0.21 \pm 0.04$ & \multicolumn{1}{l}{$0.26 \pm 0.02$} & \multicolumn{1}{l}{$0.23 \pm 0.01$} & \multicolumn{1}{l}{$0.25 \pm 0.01$} & \multicolumn{1}{l}{$0.20 \pm 0.01$} & \multicolumn{1}{l}{$--$} & \multicolumn{1}{l}{$--$} & \multicolumn{1}{l}{$--$} & \multicolumn{1}{l}{$--$} \\

\multicolumn{1}{l|}{PCN} & $0.87 \pm 0.01$ & $0.78 \pm 0.01$ & $0.67 \pm 0.01$ & $0.71 \pm 0.01$ & $0.64 \pm 0.01$ & $0.51 \pm 0.01$ & $0.46 \pm 0.01$ & $0.45 \pm 0.01$ & $0.43 \pm 0.00$ \\

\multicolumn{1}{l|}{LCN} & $0.93 \pm 0.01$ & $0.91 \pm 0.02$ & $0.86 \pm 0.02$ & $\underline{0.91 \pm 0.00}$ & $\underline{0.87 \pm 0.01}$ & $\underline{0.79 \pm 0.01}$ & $\underline{0.77 \pm 0.01}$ & $0.62 \pm 0.02$ & $\underline{0.81 \pm 0.01}$ \\

\multicolumn{1}{l|}{LCN-Redist.} & $\underline{0.97 \pm 0.01}$ & $\underline{0.96 \pm 0.01}$ & $0.82 \pm 0.02$ & $0.77 \pm 0.01$ & $0.77 \pm 0.01$ & $0.62 \pm 0.01$ & $0.59 \pm 0.01$ & $0.54 \pm 0.01$ & $\underline{0.56 \pm 0.01}$ \\

\multicolumn{1}{l|}{LCN-Mean} & $0.90 \pm 0.01$ & $0.92 \pm 0.02$ & $\underline{0.87 \pm 0.01}$ & $0.90 \pm 0.01$ & $0.84 \pm 0.01$ & $0.78 \pm 0.01$ & $0.71 \pm 0.01$ & $\underline{0.76 \pm 0.01}$ & $0.80 \pm 0.01$ \\ \hline

\multicolumn{1}{l|}{Amsterdam}   & \multicolumn{1}{c}{2}       & \multicolumn{1}{c}{3}               & \multicolumn{1}{c}{4}               & \multicolumn{1}{c}{5}               & \multicolumn{1}{c}{6}                           & \multicolumn{1}{c}{7}       & \multicolumn{1}{c}{8}       & \multicolumn{1}{c}{9}       & \multicolumn{1}{c}{10}      \\ \hline
\multicolumn{1}{l|}{GPI-LS} & $0.16 \pm 0.03$ & \multicolumn{1}{l}{$0.54 \pm 0.03$} & \multicolumn{1}{l}{$0.47 \pm 0.03$} & \multicolumn{1}{l}{$0.32 \pm 0.01$} & \multicolumn{1}{l}{$--$} & \multicolumn{1}{l}{$--$} & \multicolumn{1}{l}{$--$} & \multicolumn{1}{l}{$--$} & \multicolumn{1}{l}{$--$} \\

\multicolumn{1}{l|}{PCN} & $0.82 \pm 0.02$ & $0.69 \pm 0.01$ & $0.42 \pm 0.01$ & $0.43 \pm 0.01$ & $0.49 \pm 0.01$ & $0.42 \pm 0.01$ & $0.34 \pm 0.01$ & $0.37 \pm 0.01$ & $0.22 \pm 0.01$ \\

\multicolumn{1}{l|}{LCN} & $0.93 \pm 0.04$ & $\underline{0.88 \pm 0.01}$ & $0.45 \pm 0.07$ & $\underline{0.61 \pm 0.01}$ & $\underline{0.66 \pm 0.02}$ & $0.73 \pm 0.02$ & $\underline{0.45 \pm 0.01}$ & $0.89 \pm 0.01$ & $0.62 \pm 0.04$ \\

\multicolumn{1}{l|}{LCN-Redist.} & $0.86 \pm 0.07$ & $0.71 \pm 0.02$ & $0.36 \pm 0.03$ & $0.40 \pm 0.02$ & $0.43 \pm 0.02$ & $0.41 \pm 0.01$ & $0.34 \pm 0.01$ & $0.42 \pm 0.01$ & $0.38 \pm 0.01$ \\

\multicolumn{1}{l|}{LCN-Mean} & $\underline{0.99 \pm 0.01}$ & $0.85 \pm 0.02$ & $\underline{0.90 \pm 0.02}$ & $0.54 \pm 0.03$ & $0.47 \pm 0.02$ & $\underline{0.83 \pm 0.01}$ & $0.39 \pm 0.01$ & $\underline{0.90 \pm 0.01}$ & $\underline{0.82 \pm 0.02}$ \\ \hline

\multicolumn{10}{c}{Gini Index (the lower the better)}                                                                                                                                                                                                                                                                                                     \\ \hline
\multicolumn{1}{l|}{Xi'an}       & \multicolumn{1}{c}{2}       & \multicolumn{1}{c}{3}               & \multicolumn{1}{c}{4}               & \multicolumn{1}{c}{5}               & \multicolumn{1}{c}{6}                           & \multicolumn{1}{c}{7}       & \multicolumn{1}{c}{8}       & \multicolumn{1}{c}{9}       & \multicolumn{1}{c}{10}      \\ \hline
\multicolumn{1}{l|}{GPI-LS}      & $0.30 \pm 0.05$             & \multicolumn{1}{l}{$0.30 \pm 0.03$} & \multicolumn{1}{l}{$0.46 \pm 0.02$} & \multicolumn{1}{l}{$0.50 \pm 0.02$} & \multicolumn{1}{l}{$\underline{0.49 \pm 0.01}$} & \multicolumn{1}{l}{$--$}    & \multicolumn{1}{l}{$--$}    & \multicolumn{1}{l}{$--$}    & \multicolumn{1}{l}{$--$}    \\
\multicolumn{1}{l|}{PCN}         & $0.10 \pm 0.01$             & $0.18 \pm 0.01$                     & $0.30 \pm 0.01$                     & $0.30 \pm 0.00$                     & $0.30 \pm 0.00$                                 & $0.34 \pm 0.00$             & $0.44 \pm 0.00$             & $0.43 \pm 0.00$             & $0.46 \pm 0.00$             \\
\multicolumn{1}{l|}{LCN}         & $0.06 \pm 0.01$             & $0.09 \pm 0.01$                     & $0.19 \pm 0.01$                     & $0.25 \pm 0.01$                     & $0.23 \pm 0.01$                                 & $\underline{0.25 \pm 0.00}$ & $\underline{0.33 \pm 0.01}$ & $\underline{0.30 \pm 0.00}$ & $\underline{0.30 \pm 0.00}$ \\
\multicolumn{1}{l|}{LCN-Redist.} & $\underline{0.01 \pm 0.00}$ & $\underline{0.05 \pm 0.00}$         & $\underline{0.18 \pm 0.01}$         & $\underline{0.23 \pm 0.01}$         & $\underline{0.21 \pm 0.01}$                     & $0.27 \pm 0.01$             & $\underline{0.33 \pm 0.01}$ & $0.40 \pm 0.01$             & $0.45 \pm 0.00$             \\
\multicolumn{1}{l|}{LCN-Mean}    & $0.07 \pm 0.01$             & $0.10 \pm 0.01$                     & $0.24 \pm 0.01$                     & $0.25 \pm 0.01$                     & $0.23 \pm 0.01$                                 & $0.26 \pm 0.01$             & $0.39 \pm 0.01$             & $\underline{0.30 \pm 0.00}$ & $0.32 \pm 0.00$             \\ \hline
\multicolumn{1}{l|}{Amsterdam}   & \multicolumn{1}{c}{2}       & \multicolumn{1}{c}{3}               & \multicolumn{1}{c}{4}               & \multicolumn{1}{c}{5}               & \multicolumn{1}{c}{6}                           & \multicolumn{1}{c}{7}       & \multicolumn{1}{c}{8}       & \multicolumn{1}{c}{9}       & \multicolumn{1}{c}{10}      \\ \hline
\multicolumn{1}{l|}{GPI-LS}      & $0.54 \pm 0.05$             & \multicolumn{1}{l}{$0.64 \pm 0.02$} & \multicolumn{1}{l}{$0.62 \pm 0.01$} & \multicolumn{1}{l}{$0.70 \pm 0.01$} & \multicolumn{1}{l}{$--$}                        & \multicolumn{1}{l}{$--$}    & \multicolumn{1}{l}{$--$}    & \multicolumn{1}{l}{$--$}    & \multicolumn{1}{l}{$--$}    \\
\multicolumn{1}{l|}{PCN}         & $0.13 \pm 0.02$             & $0.48 \pm 0.01$                     & $0.58 \pm 0.01$                     & $0.67 \pm 0.01$                     & $0.63 \pm 0.01$                                 & $0.66 \pm 0.01$             & $0.71 \pm 0.01$             & $0.72 \pm 0.01$             & $0.74 \pm 0.01$             \\
\multicolumn{1}{l|}{LCN}         & $0.03 \pm 0.02$             & $\underline{0.37 \pm 0.01}$         & $\underline{0.43 \pm 0.01}$         & $\underline{0.52 \pm 0.01}$         & $\underline{0.47 \pm 0.01}$                     & $\underline{0.50 \pm 0.00}$ & $\underline{0.60 \pm 0.00}$ & $0.67 \pm 0.01$             & $\underline{0.57 \pm 0.01}$ \\
\multicolumn{1}{l|}{LCN-Redist.} & $0.06 \pm 0.04$             & $0.45 \pm 0.02$                     & $0.61 \pm 0.02$                     & $0.64 \pm 0.01$                     & $0.62 \pm 0.01$                                 & $0.64 \pm 0.01$             & $0.69 \pm 0.01$             & $0.69 \pm 0.00$             & $0.71 \pm 0.00$             \\
\multicolumn{1}{l|}{LCN-Mean}    & $\underline{0.01 \pm 0.00}$ & $0.40 \pm 0.02$                     & $0.40 \pm 0.01$                     & $0.56 \pm 0.02$                     & $0.54 \pm 0.01$                                 & $0.50 \pm 0.01$             & $0.60 \pm 0.01$             & $\underline{0.65 \pm 0.00}$ & $0.58 \pm 0.00$             \\ \hline
\end{tabular}

}
\end{table*}

\textbf{Total efficiency}: measures how effectively a generated line captures the total travel demand of a city. It is calculated as the simple sum of all elements in the value vector -- the sum of all satisfied demands for each group.

\textbf{Gini coefficient}: quantifies the reward distribution among various groups in the city. A value of 0 indicates perfect equality (equal percentage of satisfied OD flows per group), while 1 represents perfect inequality (only the demand of one group is satisfied). Although traditionally used to assess income inequality \cite{zheng_ai_2022}, it has also been employed in the context of transportation network design \cite{feng_multicriteria_2014}. In \cref{fig:full_results_all_metrics} and \cref{tbl:full_results}, we show detailed results on both Xi'an and Amsterdam for all objectives. \cref{fig:eum_lines_all} and \cref{fig:sw_lines_all} show the learning curves of PCN and LCN for all objectives.

\clearpage

\end{document}